\documentclass{article}

\usepackage{amsthm, amsmath}
\usepackage{amssymb}
\usepackage[round]{natbib}
\usepackage{colordvi}
\usepackage{graphicx}
\usepackage{color}

\newtheorem{thm}{Theorem}
\newtheorem{prop}[thm]{Proposition}
\newtheorem{cor}[thm]{Corollary}
\newtheorem{lemma}[thm]{Lemma}
\newtheorem{defn}[thm]{Definition}

\theoremstyle{plain}

\DeclareMathOperator*{\bbE}{\mathbb{E}}

\newcommand{\reals}{\mathbb{R}}

\newcommand{\fS}{\mathfrak{S}}
\newcommand{\sS}{{\mathcal S}}
\newcommand{\sG}{{\mathcal G}}
\newcommand{\sH}{{\mathcal H}}
\newcommand{\sM}{{\mathcal M}}
\newcommand{\sX}{{\mathcal X}}
\newcommand{\sZ}{{\mathcal Z}}

\newcommand{\argmin}{\operatornamewithlimits{arg\ min}}
\newcommand{\etat}{\tilde{\eta}}
\newcommand{\etatmax}{\tilde{\eta}_{\max}}
\newcommand{\etatmin}{\tilde{\eta}_{\min}}
\newcommand{\emax}{\eta_{\max}}
\newcommand{\emin}{\eta_{\min}}
\newcommand{\ind}[1]{{\bf 1}_{\{#1\}}}
\newcommand{\norm}[1]{\left\|#1\right\|}
\newcommand{\eps}{\epsilon}
\newcommand{\abs}[1]{\left\lvert #1 \right\rvert}
\newcommand{\wh}[1]{{\widehat{#1}}}
\newcommand{\wt}[1]{{\widetilde{#1}}}
\newcommand{\fP}{\mathfrak{P}}
\newcommand{\fPd}{\mathfrak{P}^2_*}

\newcommand{\clip}{\mathrm{Clip}}
\newcommand{\cfhat}{\breve{f}_n}

\newcommand{\set}[1]{\left\{#1\right\}}
\newcommand{\paren}[1]{\left(#1\right)}

\newcommand{\bcase}{\left\{ \begin{array}{ll} }
\newcommand{\ecase}{\end{array} \right. }

\newcommand{\ev}{\mathbb{E}}
\newcommand{\R}{\mathbb{R}}
\newcommand{\half}{\mbox{$\frac12$}}

\newcommand{\Pt}{\tilde{P}}

\newcommand{\qt}{\tilde{q}}
\newcommand{\Yt}{\tilde{Y}}
\newcommand{\Ptn}{\tilde{P}_0}
\newcommand{\Ptp}{\tilde{P}_1}
\newcommand{\pn}{\pi_0}
\newcommand{\pp}{\pi_1}
\newcommand{\ptn}{\tilde{\pi}_0}
\newcommand{\ptp}{\tilde{\pi}_1}
\newcommand{\htn}{\tilde{p}_0}
\newcommand{\htp}{\tilde{p}_1}
\newcommand{\Rtn}{\tilde{R}_0}
\newcommand{\Rtp}{\tilde{R}_1}
\newcommand{\pnhat}{\widehat{\pi}_0}
\newcommand{\pphat}{\widehat{\pi}_1}
\newcommand{\pnhatroc}{\widehat{\pi}_0^{\rm{roc}}}
\newcommand{\pphatroc}{\widehat{\pi}_1^{\rm{roc}}}
\newcommand{\pnhatcpe}{\widehat{\pi}_0^{\rm{cpe}}}
\newcommand{\pphatcpe}{\widehat{\pi}_1^{\rm{cpe}}}
\newcommand{\ptnhat}{\widehat{\tilde{\pi}}_0}
\newcommand{\ptphat}{\widehat{\tilde{\pi}}_1}
\newcommand{\fhat}{\widehat{f}_n}
\newcommand{\wR}{\widehat{R}}
\newcommand{\g}{\gamma}
\newcommand{\ks}{\kappa^*}

\newcommand{\step}{\mathop{u}}
\renewcommand{\a}{\alpha}
\newcommand{\wa}{\widehat{\alpha}}
\newcommand{\khat}{\widehat{\kappa}}
\newcommand{\lam}{\lambda}
\newcommand{\essinf}{\mathop{\mathrm{ess \ inf}}}
\newcommand{\esssup}{\mathop{\mathrm{ess \ sup}}}
\newcommand{\supp}{\mathop{\mathrm{supp}}}
\newcommand{\mutual}{{\bf (C)}}
\newcommand{\irreducible}{{\bf (B)}}

\title{{\bf Classification with
Asymmetric Label Noise: Consistency and Maximal Denoising}}

\author{
Gilles Blanchard\thanks{Universit\"{a}t Potsdam, Institute of Mathematics}, \ \
Marek Flaska\thanks{University of Michigan, Department of Nuclear
Engineering and Radiological Sciences}, \ \
Gregory Handy\thanks{University of Michigan, Department of Electrical
Engineering and Computer Science}, \\
Sara Pozzi\footnotemark[2], \ \
Clayton Scott\footnotemark[3]
}

\begin{document}
\maketitle

\begin{abstract}

In many real-world classification problems, the labels of training
examples are randomly corrupted. Most previous theoretical work on
classification with label noise assumes that the two classes are
separable, that the label noise is independent of the true class label, or
that the noise proportions for each class are known. In this work, we give
conditions that are necessary and sufficient for the true
class-conditional distributions to be identifiable. These conditions are
weaker than those analyzed previously, and allow for the classes to be
nonseparable and the noise levels to be asymmetric and unknown. The
conditions essentially state that a majority of the observed labels are correct
and that the true class-conditional distributions are ``mutually
irreducible," a concept we introduce that limits the similarity of the two
distributions. For any label noise problem, there is a unique pair of true
class-conditional distributions satisfying the proposed conditions, and we
argue that this pair corresponds in a certain sense to maximal denoising
of the observed distributions.

Our results are facilitated by a connection to ``mixture proportion
estimation," which is the problem of estimating the maximal proportion of
one distribution that is present in another. We establish a novel rate of
convergence result for mixture proportion estimation, and apply this to
obtain consistency of a discrimination rule based on surrogate loss
minimization. Experimental results on benchmark data and a nuclear
particle classification problem demonstrate the efficacy of our approach.




\end{abstract}

\section{Introduction}

In binary classification, one observes multiple
realizations of two different classes,
\begin{align*}
X_0^1, \ldots, X_0^m & \stackrel{iid}{\sim} P_0, \\
X_1^1, \ldots, X_1^n & \stackrel{iid}{\sim} P_1,
\end{align*}
where $P_0$ and $P_1$, the class-conditional distributions, are
probability distributions on a Borel space $(\sX, \mathfrak{S})$. The
feature vector $X_y^i \in \sX$ denotes the $i$-th realization from class
$y \in \{0,1\}$. The general goal is to construct a classifier from this
data.

There are several kinds of noise that can affect a classification problem.
A first type of noise occurs when $P_0$ and $P_1$ have overlapping
supports, meaning that the label is not a deterministic function of the
feature vector. In this situation, even an optimal classifier makes
mistakes. In this work, we consider a second type of noise, {\em label
noise}, that can occur {\em in addition to} the first type of noise. With
label noise, some of the labels of the training examples are corrupted. We
focus in particular on {\em random} label noise, as opposed to feature-dependent
or adversarial label noise.

To model label noise, we represent the training data via
contamination models:
\begin{align}
X_0^1, \ldots, X_0^m & \stackrel{iid}{\sim} \Ptn : = (1-\pn)P_0 + \pn P_1,
\label{eqn:contam0} \\
X_1^1, \ldots, X_1^n & \stackrel{iid}{\sim} \Ptp : = (1-\pp)P_1 + \pp P_0.
\label{eqn:contam1}
\end{align}
According to these mixture representations, each ``apparent"
class-conditional distribution is in fact a contaminated version of the
true class-conditional distribution, where the contamination comes from
the other class. Thus, $\tilde{P}_0$ governs the training data
with apparent class label $0$. A proportion $1 - \pn$ of these examples
have $0$ as their true label, while the remaining $\pn$ have a true label
of $1$. Similar remarks apply to $\tilde{P}_1$. The noise is
asymmetric in that $\pn$ need not equal $\pp$. We emphasize that $\pn$ and
$\pp$ are unknown. The distributions $P_0$ and $P_1$ are also unknown, and
we do not wish to impose models for them. In particular, the supports of
$P_0$ and $P_1$ may overlap, so that the classes are not separable. Previous work on classification with random label noise, reviewed below,
has not considered the problem in this generality.

Our first contribution is to
introduce necessary and sufficient conditions on the elements $P_0, P_1, \pn, \pp$ of
the contamination models such that these elements are uniquely determined given $\Ptn$ and $\Ptp$. These conditions are the following:
\begin{itemize}
\item (Total noise level) $\pn + \pp < 1$,
\item (Mutual irreducibility) It is not possible to write $P_0$ as a nontrivial mixture of
$P_1$ and some other distribution, and {\em vice versa}.
\end{itemize}
To shed some light on these conditions, we remark that in the absence of
any assumption, the solution
$(P_0,P_1,\pn,\pp)$ to \eqref{eqn:contam0}-\eqref{eqn:contam1},
when the contaminated distributions $\Ptn,\Ptp$ are given
(i.e., in the limit of infinite sample sizes, or population version of the
problem), is non-unique.
For example, were the condition on total label noise not required, for any
solution,
swapping the role of classes 0 and 1 would also be a solution (with complementary
contamination probabilities), while leaving the apparent labels unchanged.

Furthermore, we describe in detail (at the population level) the geometry of
the set of all possible solutions $(P_0,P_1,\pn,\pp)$ to
\eqref{eqn:contam0}-\eqref{eqn:contam1}.
We argue that for any pair $\Ptn \neq \Ptp$, there always exists a
{\em unique} solution satisfying the above two conditions. Moreover,
this solution uniquely corresponds to the maximum possible total label
noise level $(\pp+\pn)$ compatible with the observed contaminated
distributions, and also to the maximum possible total variation
separation $\norm{P_1-P_0}_{TV}$ under the condition $\pp + \pn <1$. In this sense,
$P_0$ and $P_1$ satisfying the second condition are {\em maximally denoised} versions of the contaminated distributions.

Our second contribution is to develop a discrimination rule that is
universally consistent in the sense that for any $\Ptn, \Ptp$, it
consistently estimates the optimal classification performance as defined
with respect to the maximally denoised distributions (which are the
underlying uncontaminated distributions under the above conditions). A key
aspect of our contribution is that the label noise proportions $\pi_0$ and
$\pi_1$ are unknown, in contrast to previous work, and the linchpin of our
solution is a method for accurately estimating $\pi_0$ and $\pi_1$. We
argue that these proportions can be estimated using methods for {\em
mixture proportion estimation} (MPE), which is the problem of estimating
the mixing proportion of one distribution in another. We review
previous work on MPE and also establish a new rate of convergence result
for MPE that is employed in our analysis.

As a third contribution, we present experimental results indicating that
the proposed methodology is practically viable. In particular, we show
that $\pi_0$ and $\pi_1$ can be accurately estimated using the same
principles guiding our theory. To illustrate this point, we examine some
standard benchmark data sets as well as a real data set from a nuclear
particle classification problem that is naturally described by our label
noise model.

Portions of this work appeared earlier in \citet{scott13colt} and
\citet{scott15aistats}. This longer version integrates those versions and
extends them by establishing the necessity of the proposed conditions, a
consistency analysis featuring clippable losses, a connection to class
probability estimation, and a more thorough literature review.

\subsection{Motivating Application}
\label{sec:motiv}

This work is motivated by a nuclear particle classification problem that
is critical for nuclear nonproliferation and nuclear safeguards. An
organic scintillation detector is a device commonly used to detect
high-energy neutrons. When a particle interacts with the detector, the
energy deposited by the particle is converted to a pulse-shaped voltage
waveform, which is then digitally sampled to obtain a feature vector $X
\in \R^d$, where $d$ is the number of digital samples.  The energy
distribution of detected neutrons is characteristic of the nuclear source
material, and these energy distributions can be inferred from the heights
of the observed pulses. However, these detectors are also sensitive to
gamma rays, which are frequently emitted by the same fission events that
produce neutrons, and which are also strongly present in background
radiation. Therefore, to render organic scintillation detectors useful for
characterization of nuclear materials, it is necessary to classify between
neutron and gamma-ray pulses, a problem referred to as pulse shape
discrimination (PSD) \citep{adams78,ambers11}.

Unfortunately, even in controlled laboratory settings, it is very
difficult to obtain pure samples of neutron and gamma-ray pulses. As
previously mentioned, the fission events that produce neutrons also yield
gamma rays, and gamma rays also arrive from background radiation.
Although pure gamma-ray sources do exist, when collecting measurements
from such sources, neutrons from the background cannot be completely
eliminated. If we view gamma-rays as class 0, by taking a strong and pure
gamma-ray source, $\pn$ will be small but nonzero.  On the other hand, the
proportion of gamma-rays emitted during fission is intrinsic to the source
material, and cannot be changed. Thus $\pp$ could be in the neighborhood
of one-half. With additional time-of-flight information, this proportion
can be reduced, but is still non-negligible \citep{ambers11}. Thus, PSD is
naturally described by the proposed label noise model. We study this
problem empirically in Section \ref{sec:exp}.

\subsection{Label Flipping Model for Label Noise}

Random label noise can also be modeled according to the
label flipping probabilities
$$
\mu_i := \Pr(\Yt = 1-i \, | \, Y = i).
$$
In the label flipping model, a ``clean" training data set is corrupted by
flipping the labels according to $\mu_0$ and $\mu_1$, independent of $X$.
If we assume $Y$ and $\Yt$ are jointly distributed, then
$\pi_i$ and $\mu_i$ are related via Bayes' rule. The preferred
perspective for a given label noise problem, contamination or label
flipping, is application dependent. For example, the contamination
model better
suits the nuclear particle classification problem described above. We also
find it more natural to discuss identifiability in terms of the
contamination model.

\subsection{Related Work}
\label{sec:related}

Classification in the presence of label noise has drawn the attention of
numerous researchers \citep{frenay14}. One common approach is to assume
that corrupted labels are more likely to be associated with outlying data
points. This has inspired methods to clean, correct, or reweight the
training data \citep{brodley99, rebbapragada07}, as well as the use of
robust (usually nonconvex) losses \citep{mason00, schuurmans06,
vasconcelos08, ding10, denchev12}. The above approaches are not
necessarily based on a random label noise model, but rather assume that
noisy labels are more common near the decision boundary.

Generative models have also been applied in the context of random label
noise. These impose parametric models on the data-generating
distributions, and include the label noise as part of the model.  The
parameters are then estimated using an EM algorithm \citep{bouveyron09}.
The method of \cite{lawrence01} employs kernels in this approach, allowing
for the modeling of more flexible distributions.


Negative results for convex risk minimization in the presence of label
noise have been established by \citet{long10} and \citet{manwani11}. These
works demonstrate a lack of noise tolerance for boosting and empirical
risk minimization based on convex losses, and suggest that
any approach based on convex risk minimization will require modification
of the loss, such that the risk minimizer is the optimal classifier with
respect to the uncontaminated distributions. Along these lines,
\citet{stempfel09, tewari13} recently developed such algorithms based on
convex losses. The works, however, assume knowledge of the label noise proportions.
In the sequel, we establish a consistent discrimination rule that does not assume knowledge
of $\pi_0$ and $\pi_1$; in fact the main focus of the present work is
on the estimation of those quantities.

Recently \citet{yang12icml} established performance guarantees for multiple
kernel learning with noisy labels. This work does not assume label noise
is independent of the feature vector, but does require knowledge of the
total amount of label noise.

Classification with random label noise has also been studied in the PAC
literature. Most PAC formulations assume that (i) $P_0$ and $P_1$ have
non-overlapping support (i.e., there is a deterministic ``target concept"
that provides the true labels), (ii) the label noise is symmetric (i.e.,
independent of the true class label), and (iii) the performance measure is
the probability of error \citep{angluin88, kearns93, aslam96,
cesabianchi97, bshouty98, kalai03}. Under these conditions, it typically
suffices to train on the contaminated data; only the sample complexity
changes. The case of asymmetric label noise was addressed by
\citet{blum98} under condition (i), as the basis of co-training. Some new
directions and a thorough review of this body of work were recently
presented in \cite{jabbari10}. As we discuss in the next section, new
challenges emerge when conditions (i), (ii), and (iii) are not assumed.

To our knowledge, previous work under the asymmetric noise model has not
addressed a minimal set of conditions for either consistent classification
or for consistent estimation of the label noise proportions.

Classification with label noise is related to several other machine
learning problems. When $\pp = 0$, we have ``one-sided" label noise, and
the problem reduces to learning from positive and unlabeled examples
(LPUE), also known as semi-supervised novelty detection (SSND); see
\citet{blanchard10} for a review of this literature. In particular,
\citet{blanchard10} develop theory for ``mixture proportion estimation"
that we leverage in our analysis.

A basic version of multiple instance
learning can be reduced to classification with one-sided label noise
\citep[see][]{sabato12}. In multiple instance learning, the learner is
presented with bags of instances. In one basic setting, the bags are
labeled negative if they contain only negative instances, and positive if
they contain at least one positive instance. If one assumes that the
instances in positive bags follow a mixture model $\Ptp = (1-\pi) P_1 +
\pi P_0$, and the instances are iid according to $P_0$ or $\Ptp$, the
setting is that of one-sided label noise.

As mentioned above, classification with label noise is the basis of
co-training \citep{blum98}, which is a framework for classifying instances
that are represented by two distinct ``views." The original analysis of
co-training considers the ``realizable" case, where labels are a
deterministic function of inputs. Our results allow us to state a result
for co-training without making this restrictive assumption. This result is
presented in Section \ref{sec:cotrain}.

There is also a connection between classification with label noise and
class probability estimation. As pointed out in our initial technical
report \citep{scott13tr}, there is a simple way to express mutual
irreducibility in terms of the class probability function. From this
relationship, and given other developments in this paper, it is
straightforward to express $\pi_0$ and $\pi_1$ in terms of the maximum and
minimum values of the contaminated class probability function. This
suggests an alternative estimation strategy for the label noise
proportions, which has recently been investigated by \citet{liutao16} and
\citet{menon15icml}. In Section \ref{sec:cpe}, we elaborate on this
approach and connections to these works. We also investigate this
approach experimentally in Section \ref{sec:exp}.

As a final connection with existing literature, we note that an
alternative way to view the contamination model
\eqref{eqn:contam0}-\eqref{eqn:contam1} is to interpret it as a {\em
source separation} problem. In the usual source separation setting, the
{\em realizations} from the different sources are linearly mixed, whereas
in the present model, the {\em source probability distributions} are (we
do not observe a signal superposition, but a signal coming randomly from
one or the other source). As a common point with the source separation
setting, it is necessary to postulate additional constraints on the
sources in order to resolve non-uniqueness of the possible solutions. In
independent component analysis, for instance, sources are assumed to be
independent. Our assumption of mutual irreducibility between the sources
plays a conceptually comparable role here. Similarly, the assumption on
the total noise level resolves the ambiguity that the sources would
otherwise only be identifiable up to permutation.

\subsection{Some Initial Notation}

Let $f:\sX \to \{0,1\}$ be a classifier.
Denote the (uncontaminated) Type I and Type II errors
\begin{align*}
R_0(f) &:= P_0(f(X)=1) \\
R_1(f) &:= P_1(f(X)=0).
\end{align*}
These quantities are what define many classification performance measures
of
interest, such as the so-called {\em minmax} criterion, $R(f) =
\max\{R_0(f), R_1(f)\}$,
or the probability of error, $R(f) = \nu R_1(f) + (1-\nu) R_0(f)$, where $\nu$ is the a priori probability of class 1.

We also define the corresponding contaminated Type I and II errors
\begin{align}
\Rtn(f) &:= \Ptn(f(X)=1) \nonumber \\
&= (1-\pn) R_0(f) + \pn (1 - R_1(f)) \label{eqn:r0t} \\
\Rtp(f) &:= \Ptp(f(X)=0) \nonumber \\
&= (1-\pp) R_1(f) + \pp (1 - R_0(f)) \label{eqn:r1t}.
\end{align}
These quantities can easily be estimated from the training data by their basic empirical counterparts.

\subsection{Outline}

The remainder of the paper is outlined as follows. Section
\ref{sec:challenge} discusses the challenges posed by label noise for
classifier design. Section \ref{sec:alternate} presents an alternate
representation of the contamination models that reduces the problem to
that of mixture proportion estimation, which is discussed in Section
\ref{sec:mixture}. In Section \ref{sec:mutual} we introduce our proposed
identifiability conditions, establish their sufficiency and necessity, and
also discuss maximal denoising. A method for mixture proportion estimation
is discussed in Section \ref{sec:mpe}, where a novel rate of convergence
result is presented and subsequently applied to develop a consistent
discrimination rule in Section \ref{sec:consist}. In Section
\ref{sec:cotrain}, we apply our label noise results to generalize an
earlier result on co-training. Section \ref{sec:cpe}
makes a connection between our label noise framework and the problem of
class probability estimation. Algorithm implementations are described in
Section \ref{sec:imp}, and experimental results are provided in
Section \ref{sec:exp}. Shorter proofs tend to be found in the body of the
paper, while longer ones are in an appendix.

\section{The Challenge of Label Noise}
\label{sec:challenge}

Before delving into more technical matters, we first offer an overview of the challenges posed by label noise. We focus
on the population setting ($n_0, n_1 = \infty$) and compare classifier design
based on the contaminated distributions, $\Ptn$ and $\Ptp$, versus the
true ones, $P_0$ and $P_1$. To begin, we introduce the following condition on the
total amount of label noise.
\begin{description}
\item[(A)] $\pn + \pp < 1$.
\end{description}
This condition states, in a certain sense, that a majority of
the labels are correct on average. It even allows that one of the
proportions be very close to one if the other proportion is small enough.
This condition was previously adopted by \cite{blum98}.

Let $p_0$ and $p_1$ be densities of $P_0$ and $P_1$, respectively, with
respect to a common dominating measure. Then
\begin{align*}
\htn(x) &:= (1-\pn) p_0(x) + \pn p_1(x), \\
\htp(x) &:= (1-\pp) p_1(x) + \pp p_0(x),
\end{align*}
are respective densities of $\Ptn$ and $\Ptp$.
\begin{prop}
\label{prop:p1}
Assume {\bf (A)} holds. For all $\g \ge 0$, and every $x$ such
that $p_0(x) > 0$ and $\htn(x) > 0$,
\begin{equation*}
\frac{p_1(x)}{p_0(x)} > \g \iff \frac{\htp(x)}{\htn(x)} > \lam,
\end{equation*}
where
\begin{equation}
\label{eqn:gamlam}
\lam = \frac{\pp + \g (1-\pp)}{1 - \pn + \g \pn}.
\end{equation}
\end{prop}
The proof involves a sequence of simple algebraic steps to transform one
likelihood ratio into the other, and the use of {\bf (A)} to ensure that
the direction of the inequality is preserved.

For most performance measures of interest (probability of error,
Neyman-Pearson, etc.), it is well-known that the optimal classifier takes
the form of a likelihood ratio test (LRT) based on the true densities
\citep{lehmann86, naga14nips}.
According to the proposition, every true LRT is identical to a
contaminated LRT with a different threshold.  As the threshold of one LRT
sweeps over its range, so too does the threshold of the other LRT.
Equivalently, both LRTs generate the same receiver operating
characteristic (ROC).

However, if we design a classifier with respect to the contaminated
estimates of performance, we will not obtain a classifier
that is optimal with respect to the true performance measure, except in
very special circumstances. To make this point concrete, we now consider
four specific performance measures.


{\bf Probability of error.} When the feature vector $X$ and label $Y$ are
jointly distributed, the probability of misclassification is minimized by
a LRT, where the threshold $\gamma$ is given by the ratio of {\em a
priori} class probabilities. If $\gamma = 1$, then the corresponding
threshold for the contaminated LRT is also 1, regardless of $\pn$ and
$\pp$, which follows directly from \eqref{eqn:gamlam}. Furthermore,
assuming $\pi_0, \pi_1 > 0$ and with some simple algebra it is easy to
show that $\lambda = \gamma$ only if $\gamma = 1$. Thus, treating the
contaminated data as if it were clean is suboptimal whenever the a priori
class probabilities are unequal.



{\bf Neyman-Pearson.} As noted above, the true and contaminated LRTs have
the same ROC. If a point on this ROC is chosen such that $\Rtn(f) =
\alpha$, it will generally not be the case that $R_0(f) = \alpha$. This
follows because $\Rtn(f) = (1-\pn) R_0(f) + \pn R_1(f)$. Simple algebra
shows that $R_0(f) = \Rtn(f)$ iff $\pn = 0$ or $R_0(f) + R_1(f) = 1$. The
latter condition is not satisfied by an optimal classifier unless $P_0 =
P_1$, since it corresponds to random guessing. The former case, $\pn = 0$,
means the negative class has no contamination, and is equivalent (after
swapping class labels) to learning from positive and unlabeled examples.

{\bf Minmax.} The minmax criterion is defined as $R(f) :=
\max\{R_0(f),R_1(f)\}$, and the minmax classifier is the minimizer of this
quantity. The minmax classifier corresponds to the point on the ROC of the
true and contaminated LRTs where $R_0(f) = R_1(f)$. Indeed, if $R_0(f) \ne
R_1(f)$, then $\max \{R_0(f), R_1(f)\}$ can be decreased by moving along
the ROC such that the larger of $R_0(f), R_1(f)$ is decreased. Thus,
designing a classifier with respect to the contaminated distributions
yields a point on the optimal ROC where $\Rtn(f) = \Rtp(f)$. Using
equations \eqref{eqn:r0t} and \eqref{eqn:r1t}, simple algebra reveals that
$\Rtn(f) = \Rtp(f)$ and $R_0(f) = R_1(f)$ for the same $f$ iff $\pn = \pp$
or $R_0(f) = R_1(f) = \frac12$. The first condition is not satisfied for
asymmetric label noise, and the latter condition is not true for an
optimal classifier unless $P_0 = P_1$.

{\bf Balanced Error.} \citet{menon15icml} actually show that the balanced
error, given by $\frac12 (R_0(f) + R_1(f))$, is the only performance
measure that is a function of $R_0(f)$ and $R_1(f)$, such that optimizing
the corrupted performance measure is equivalent to optimizing the
clean performance measure regardless of the label noise proportions or
prior class probabilities.


In summary, a classifier that is optimal with respect to a contaminated
performance measure is not optimal for the uncontaminated performance
measure except in special cases. Accurate estimation of the true
performance measure is thus a critical issue for classifier design. In the
next section, we expose a technique for estimating performance using
estimates of the label noise proportions.

\section{Alternate Contamination Model}
\label{sec:alternate}

We introduce an alternate contamination model that will later be used to obtain estimates of the label noise proportions, and consequently estimates of classifier performance.
\begin{lemma}
\label{le:le1}
If $P_0 \ne P_1$ and {\bf (A)} holds, then $\Ptp \neq \Ptn$,
and there exist unique $0 \le \ptn, \ptp < 1$ such that
\begin{align}
\Ptn & = (1-\ptn) P_0 + \ptn \Ptp \label{eqn:ssnd0} \\
\Ptp & = (1-\ptp) P_1 + \ptp \Ptn. \label{eqn:ssnd1}
\end{align}
In particular $\ptn = \frac{\pn}{1-\pp} < 1$ and $\ptp = \frac{\pp}{1-\pn} < 1$.
\end{lemma}
\begin{proof}
To see that  $\Ptp \neq \Ptn$, assume that equality holds.
Plugging in \eqref{eqn:contam0}-\eqref{eqn:contam1}, we obtain
\[
(1-\pp - \pn)P_1 = (1-\pp - \pn)P_0,
\]
which, since $P_0 \neq P_1$, would imply $\pp + \pn = 1$ and contradict {\bf (A)}.

We turn to  identity \eqref{eqn:ssnd0}. Matching distributions,  the identity holds iff
\begin{align*}
P_1 (\pn - \ptn (1-\pp)) & = P_0 (1 - \ptn + \pp\ptn - (1 - \pn)) \\
& = P_0 (\pn - \ptn (1-\pp)).
\end{align*}
Since $P_0 \ne P_1$, the unique solution is $\ptn = \frac{\pn}{1-\pp}$. From {\bf (A)} it follows that $\ptn < 1$.
Similar reasoning applies to the second identity.
\end{proof}

This lemma motivates estimates of the true Type I and Type II errors.
For any classifier $f$, we may express the contaminated
Type I and Type II errors as
\begin{eqnarray}
\tilde{R}_0(f) &=& \tilde{P}_0(f(X)=1) \nonumber \\
                  &=& (1-\tilde{\pi}_0)R_0(f) + \tilde{\pi}_0(1-\tilde{R}_1(f))  \label{eq:noisyz} \\
\tilde{R}_1(f) &=& \tilde{P}_1(f(X)=0) \nonumber \\
                  &=& (1-\tilde{\pi}_1)R_1(f) + \tilde{\pi}_1(1-\tilde{R}_0(f)), \label{eq:noisyo} \
\end{eqnarray}
where Equations (\ref{eq:noisyz}) and (\ref{eq:noisyo}) follow from
Lemma~\ref{le:le1}. By solving for $R_0(f)$ and $R_1(f)$ in
(\ref{eq:noisyz}) and (\ref{eq:noisyo}), we find
\begin{eqnarray}
R_0(f) = \frac{\tilde{R}_0(f) - \tilde{\pi}_0(1-\tilde{R}_1(f))}{1-\tilde{\pi}_0} &=&
1-\tilde{R}_1(f) - \frac{1-\tilde{R}_0(f) - \tilde{R}_1(f)}{1-\tilde{\pi}_0}
\label{eq:noisyz2} \\
R_1(f) = \frac{\tilde{R}_1(f) - \tilde{\pi}_1(1-\tilde{R}_0(f))}{1-\tilde{\pi}_1} &=&
1-\tilde{R}_0(f) - \frac{1-\tilde{R}_1(f) - \tilde{R}_0(f)}{1-\tilde{\pi}_1}.
\label{eq:noisyo2} \
\end{eqnarray}
We can estimate $\tilde{R}_0(f)$ and $\tilde{R}_1(f)$ from the training
data. Therefore, if we can estimate $\ptn$ and $\ptp$, then we can
estimate $R_0(f)$ and $R_1(f)$, and thereby design a classifier. This
approach was analyzed in \citet{scott13colt}. In
Sec. \ref{sec:consist} we describe another approach to classifier design
based on surrogate loss minimization that also relies on estimates of
$\ptn$ and $\ptp$. In the
next section we describe a framework that is used to estimate $\ptn$
and $\ptp$.



We conclude this section with a converse to Lemma~\ref{le:le1}:
\begin{lemma}
\label{le:le1conv}
Assume that \eqref{eqn:ssnd0}-\eqref{eqn:ssnd1} hold and $\Ptp\neq\Ptn$. Then $P_1 \neq P_0$ and
there exist unique $\pp,\pn \in [0,1)$
(namely $\pi_0 = \frac{\ptn(1-\ptp)}{1-\ptp\ptn}$ and $\pi_1 = \frac{\ptp(1-\ptn)}{1-\ptp\ptn} $) so that
\eqref{eqn:contam0}-\eqref{eqn:contam1} hold; furthermore, {\bf (A)} is satisfied.
\end{lemma}
\begin{proof}
Assume \eqref{eqn:ssnd0}-\eqref{eqn:ssnd1} hold. Since we assume $\Ptp\neq \Ptn$, it holds that $\ptp,\ptn<1$.
To see that  $P_0 \neq P_1$, assume that equality holds.
Plugging in \eqref{eqn:ssnd0}-\eqref{eqn:ssnd1} and after straightforward manipulation, we obtain equivalently
\[
\frac{1-\ptp \ptn}{(1-\ptp)(1-\ptn)} \Ptp = \frac{1-\ptp \ptn}{(1-\ptp)(1-\ptn)} \Ptn,
\]
which would contradict the assumption $\Ptp \neq \Ptn$.

Next, in order for identity~\eqref{eqn:contam0} to hold, by matching
distributions in a similar way as in the proof of Lemma~\ref{le:le1}, we
arrive at the equivalent relation $(\ptn(1-\pi_1)-\pi_0)\Ptn=
(\ptn(1-\pi_1)-\pi_0)\Ptp$. Since $\Ptp\neq\Ptn$, the unique solution is
$\pi_0 = \ptn(1-\pi_1)$. Similarly, for~\eqref{eqn:contam1} to hold the
unique solution is $\pi_1 = \ptp(1-\pi_0)$. From these we derive the
announced expression for $\pi_0,\pi_1$. It is then easy to check that
$\pi_0+\pi_1-1=-\frac{(1-\ptp)(1-\ptn)}{1-\ptp\ptn}<0$, so that {\bf (A)}
holds.
\end{proof}

Together, Lemmas~\ref{le:le1} and~\ref{le:le1conv} imply that for known, distinct uncontaminated
distributions $P_0\neq P_1$,
there is an explicit one-to-one correspondence
between the contamination proportions $(\pp,\pn)$ of the initial
contamination models
\eqref{eqn:contam0}-\eqref{eqn:contam1} under constraint {\bf (A)},
and the proportions $(\ptp,\ptn)$ in the
representation \eqref{eqn:ssnd0}-\eqref{eqn:ssnd1} (with the only constraint
$0\leq \ptp,\ptn <1$).


The alternate representations \eqref{eqn:ssnd0}-\eqref{eqn:ssnd1} are {\em
decoupled} in the sense that \eqref{eqn:ssnd0} does not involve $P_1$, while
\eqref{eqn:ssnd1} does not involve $P_0$. This allows us to estimate $\ptn$
and $\ptp$ separately, by reducing to the problem of ``mixture proportion estimation" (see next section). It further motivates the mutual irreducibility condition on
$(P_0,P_1)$ that, together with {\bf (A)}, ensures that $\ptn, \ptp$ are
identifiable. The decoupling perspective also allows us to address the
following question: Given the contaminated distributions $\Ptp,\Ptn$, while
$(P_0,P_1)$ are unknown, what are the solutions $(\pi_0,\pi_1,P_0,P_1)$
satisfying model \eqref{eqn:contam0}-\eqref{eqn:contam1}? Obviously,
$(0,0,\Ptp,\Ptn)$ is a trivial solution; we will argue that mutual
irreducibility ensures that the solution is unique and non-trivial,
and furthermore that the resulting $P_0, P_1$ correspond to maximally denoised
versions of $\Ptp,\Ptn$. The issues are developed in Section \ref{sec:mutual}. In the next section, we review the work of \citet{blanchard10}.


%
%
%
%

\section{Irreducibility and Mixture Proportion Estimation}
\label{sec:mixture}

Let $F$, $G$, and $H$ be distributions on $(\sX, \mathfrak{S})$ such that
\begin{equation}
F = (1-\kappa) G + \kappa H, \nonumber
\end{equation}
where $0 \le \kappa \le 1$. Mixture proportion estimation is the following
problem:
given iid realizations from both $F$ and $H$, estimate $\kappa$. This
problem was previously addressed by
\cite{blanchard10}, and here we relate the essential definitions and
results from that work.

Without additional assumptions, $\kappa$ is not an identifiable parameter, as
noted by \citet{blanchard10}. In particular, if $F = (1-\kappa) G +\kappa H$\,
holds, then any alternate decomposition of the form $F = (1-\kappa+\delta) G'
+ (\kappa-\delta) H $\,, with $G' = (1-\kappa+\delta)^{-1}((1-\kappa) G + \delta
H)$\,, and $\delta \in [0,\kappa)$\,, is also valid. Because we have no
direct knowledge of $G$\,, we cannot decide which representation is the
correct one. Therefore, to make $\kappa$ identifiable, some additional
condition must be assumed. The following definition will be
useful.

\begin{defn}
Let $G$\,, $H$ be probability distributions. We say that $G$ is {\em irreducible} with respect to $H$ if there exists no decomposition of the form
$G = \gamma H + (1-\gamma) F' $, where $F'$ is some probability
distribution and $0< \gamma \leq 1$\,. We say that $G$ and $H$ are {\em mutually irreducible} if G is
irreducible with respect to H and vice versa.
\end{defn}


The following was established by \citet{blanchard10}.
\begin{prop}
\label{prop:canondecmp}
Let $F$\,, $H$ be probability distributions.
If $F \neq H$, there is a unique $\kappa^*\in[0,1)$ and $G$ such that the
decomposition $F = (1-\kappa^* ) G+ \kappa^* H$ holds, and such that $G$ is
irreducible with respect to $H$\,. If we additionally
define $\kappa^*=1$ when $F = H$, then in all cases
\begin{equation}
\label{eq:nustar}
\kappa^*
= \max\{\alpha \in[0,1]: \exists \, G' \text{ probability
    distribution: } F = (1-\alpha)G' + \alpha H \}\,. \nonumber
\end{equation}
\end{prop}

By this result, the following is well-defined.
\begin{defn}
For any two probability distributions $F$, $H$, define
$$
\kappa^*(F|H) := \max\{\alpha \in[0,1]: \exists \, G' \text{ probability
     distribution: } F = (1-\alpha)G' + \alpha H \},
$$
the maximal proportion of $H$ in $F$.
\end{defn}
Clearly, $G$ is irreducible with respect to $H$ if and only if
$\kappa^*(G|H) = 0$. It is also interesting to note that $1-\ks(F|H)$ is
an example of a statistical distance. That is, $1-\ks(F|H)$ is
always nonnegative, and is equal to zero if and only if $F=H$, by
Proposition \ref{prop:canondecmp}. Furthermore, Proposition \ref{prop:ks}
below states that this distance can be expressed in terms of the
likelihood ratio, like Kullback-Liebler and other information divergences.
This statistical distance has been studied previously for discrete
distributions in the analysis of Markov chains \citep{aldous87}, where it
is called the ``separation distance." In general, $\ks(F|H) \ne \ks(H|F)$,
so that this is not actually a metric on distributions.

To consolidate the above notions, we state the following
corollary which expresses that irreducibility of $G$ with respect to $H$
is sufficient for the mixture proportion to be identifiable.
\begin{cor}
\label{cor:irrd}
If $F = (1-\gamma) G + \gamma H$, and $G$ is irreducible with respect to $H$, then
$\gamma = \kappa^*(F|H)$.
\end{cor}

Some intuition for $\ks$ and irreducibility come from the following
result. Part of the result is in terms of the
receiver operating
characteristic (ROC) for the problem of testing the null hypothesis
$X \sim H$ against the alternative
$X \sim F$. Given a
measurable set $S \in \fS$, we can think of $S$ as a rejection
region (where the null hypothesis is rejected).
Then $H(S)$ is the false positive rate and $F(S)$ is the true positive
rate, and the optimal ROC is defined as
$$
\beta(\tau) := \sup \{F(S) \, | \, H(S) \le \tau, S \in \fS \}.
$$
The ensuing result follows from Theorem 6 of \citet{blanchard10}.
\begin{prop}[\citet{blanchard10}]
\label{prop:ks}
$$
\kappa^{*}(F|H)=\inf_{S \in \fS, H(S) > 0} \, \frac{F(S)}{H(S)}
=\mathop {\inf }\limits_{\tau \in
[0,1)}{\left\{\dfrac{1-\beta(\tau)}{1-\tau}\right\}}.
$$
If $f$ and $h$ are densities of $F$ and $H$, respectively, with respect
to a common dominating measure, then
\begin{equation}
\label{eqn:nulrt}
\ks(F|H) = \essinf_{x  \in \supp(H)} \frac{f(x)}{h(x)}.
\end{equation}
\end{prop}
\begin{proof}
The first two identities are established by \citet{blanchard10}. See also
\cite{scottWSLnotes}. The proof of the first identity is very similar to
the
proof of the third identity given below. Intuition for the second identity
comes from the first identity and the observation that the optimal ROC is
concave.
To prove the third identity, let
$$
\g^* = \essinf_{x  \in \supp(H)} \frac{f(x)}{h(x)}.
$$
We need to show (i) $\exists g$ such that $f = (1-\g^*)g + \g^* h$, and
(ii) if $\g > \g^*$, then no such $g$ exists. To see (i), take $g = (f -
\g^* h)/(1-\g^*)$, which clearly integrates to one, and is a.s. nonnegative by
definition of $\g^*$.  To see (ii), suppose that for some $\g > \g^*$,
there exists a density $g$ with $f = (1-\g)g + \g h$. Then for
all $x$ such that $h(x) > 0$,
$$
\frac{f(x)}{h(x)} = \g + (1-\g) \frac{g(x)}{h(x)} \ge \g > \g^*,
$$
which contradicts the definition of $\g^*$.
\end{proof}
An alternate proof of the last statement, based on properties of ROC
curves of likelihood ratio tests, is given in an appendix.

The result $\ks(F|H)=\inf_{S \in \fS, H(S) > 0} \, \frac{F(S)}{H(S)}$
motivates the universally consistent estimator of $\ks$ due to
\cite{blanchard10}, reviewed below in Section \ref{sec:mpe}. The second
identity, which states that $\ks(F|H)$ is the slope of the optimal ROC at
its right end-point, motivates a more practical estimator
discussed in Section \ref{sec:exp}.

Proposition \ref{prop:ks} makes it possible to check irreducibility for
certain distributions. For example, $\ks(G|H)=0$ whenever the support of
$G$ does not contain the support of $H$. Irreducibility is also possible
even if $G$ and $H$ have the same support, as in the case where $G$ and
$H$ are Gaussian distributions with different means, and the variance of
$H$ is no more than the variance of $G$. This follows easily from the
density ratio characterization of $\ks$.

Proposition \ref{prop:ks} also makes it easy to check {\em mutual}
irreducibility for various distributions $P_0$ and $P_1$. Indeed, two
continuous distributions are mutually irreducible iff the
(essential) infimum and supremum of their density ratio are $0$ and
$\infty$,
respectively. Figure \ref{fig:mutual} shows three examples where $\sX =
\R$. In the first example, $P_0$ and $P_1$ are such that the support of
one is not contained in the support of the other, and therefore
mutual irreducibility is satisfied.
In the second example, $P_0$ and $P_1$ are Gaussian
distributions with equal variances and unequal means. By plugging in the
formulas for the Gaussian densities, it is easy to verify that
mutual irreducibility again holds.
In the third example, $P_0$ and $P_1$ are again
Gaussian densities with unequal means, but this time with unequal
variances. In this case, it is again not hard to show that $\ks(P_0|P_1)
= 0$, but $\ks(P_1|P_0) > 0$, where $P_1$ has the larger variance. Thus,
mutual irreducibility
does not hold in this case. We do note, however, that
$\ks(P_1|P_0)$ tends to zero very fast as the means move apart.

\begin{figure}
\centering
\includegraphics[trim=0 175 0 0, clip, width=\textwidth]{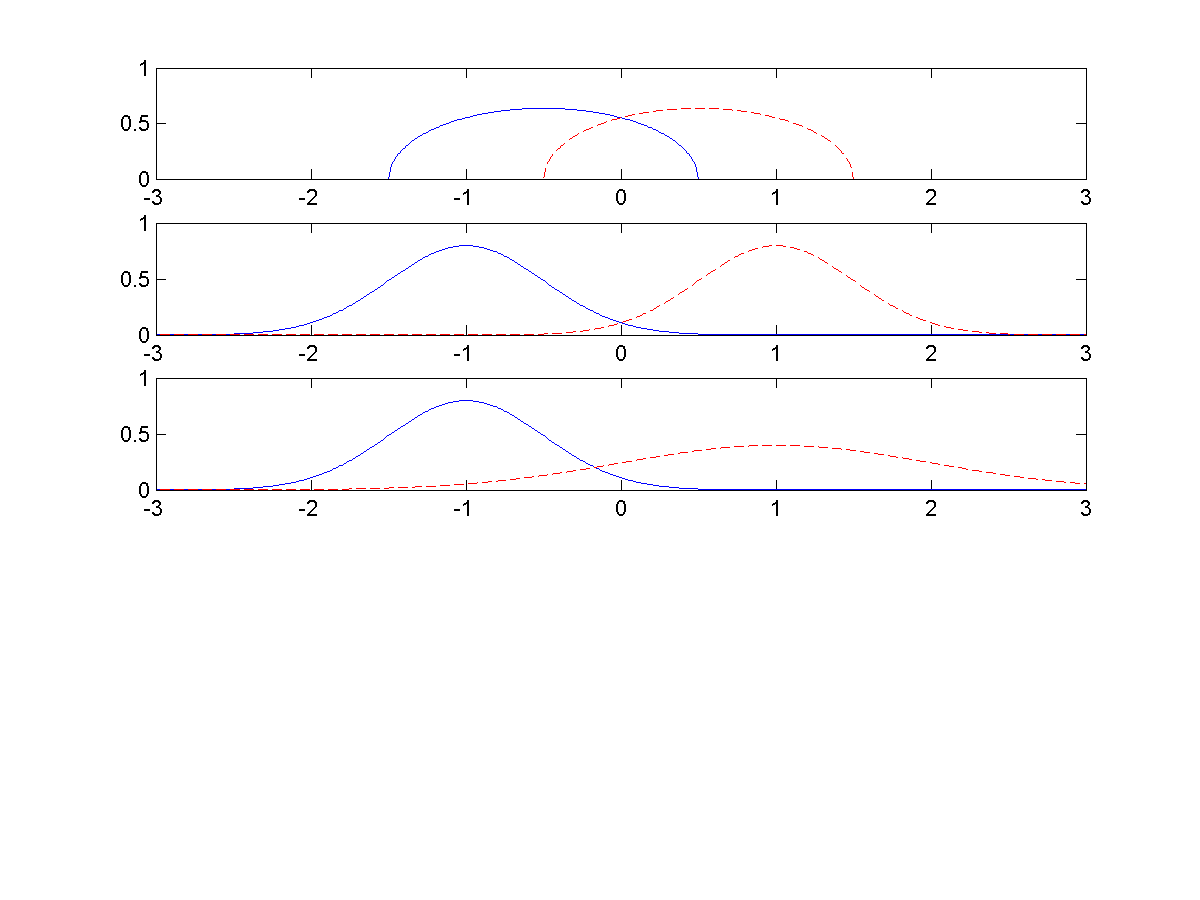}
\caption{Three one-dimensional examples that illustrate assumption
\mutual. In each example (row), $P_0$ is on the left (solid line) and
$P_1$ on the right (dotted line). In the first two examples, 
mutual irreducibility holds,
but in the third example it does not. See text for details.
\label{fig:mutual}}
\end{figure}

\section{Mutual Irreducibility:
  Sufficiency, Necessity, and Maximal Denoising}
\label{sec:mutual}


We argue that mutual irreducibility of $P_0$ and $P_1$ is both necessary
and sufficient for identifiability of the elements $(\pi_0, \pi_1, P_0,
P_1)$ of the contamination models, and relate it to the notion of maximal
denoising of the contaminated distributions. Since our focus in this
section is identifiability and not estimation, our discussion is at the
population level.

\subsection{Sufficiency of Mutual Irreducibility for Identifiability}
\label{sec:suff}

Recalling the result
of Lemma~\ref{le:le1}, the distributions $\Ptn$ and $\Ptp$ can be written
\begin{align}
\Ptn & = (1-\ptn) P_0 + \ptn \Ptp  \nonumber \\
\Ptp & = (1-\ptp) P_1 + \ptp \Ptn. \nonumber  \
\end{align}
By Corollary~\ref{cor:irrd},
we can identify $\tilde{\pi}_0$ and $\tilde{\pi}_1$ provided the following
condition holds:
\begin{description}
\item[\irreducible] $P_0$ is irreducible with respect to $\Ptp$ and $P_1$
is irreducible with respect to $\Ptn$.
\end{description}

We prefer an irreducibility assumption based on the true class-conditional
distributions, and so introduce the following:
\begin{description}
\item[\mutual] $P_0$ and $P_1$ are mutually irreducible.
\end{description}
Note that it follows from assumption \mutual \, that $P_0 \ne P_1$, which
is a hypothesis of Lemma \ref{le:le1}.
We now establish that \mutual\, and \irreducible\, are essentially equivalent.
\begin{lemma}
\label{le:le2}
$P_0$ is irreducible with respect to $\Ptp$ if and only if
$P_0$ is irreducible with respect to $P_1$ and $\pi_1<1$.
The same statement holds when exchanging the roles of the two classes.
In particular, under assumption {\bf (A)}, \mutual \, is equivalent to \irreducible \,.
\end{lemma}
\begin{proof}
This will be a proof by contraposition. Assume first that $P_0$ is not irreducible
with respect to $\Ptp$. Then there exists a probability distribution
$Q'$ and $0<\gamma\le 1$ such that
\begin{eqnarray}
& P_0 = \gamma  \Ptp + (1-\gamma) Q'. \nonumber
\end{eqnarray}
Now, plugging in Equation (2) for $\Ptp$ yields
\begin{eqnarray}
& P_0 = \gamma ((1-\pi_1)P_1 + \pi_1 P_0)+(1-\gamma)Q'. \nonumber
\end{eqnarray}
Solving for $P_0$ produces
\begin{eqnarray}
& P_0 = (1-\beta)Q' + \beta P_1, \nonumber
\end{eqnarray}
where $\beta = \gamma (\frac{1- \pi_1}{1-\gamma \pi_1})$.
Now, in the case where $\pi_1<1$, then $1-\gamma \pi_1 > 0$, and $\gamma - \gamma \pi_1 > 0$.
Since $0 < \gamma \le 1$, we deduce $0 < \beta \le 1$,
so that $P_0$ is not irreducible with respect to $P_1$.


Conversely, assume
by contradiction that $P_0$ is not irreducible with respect to $P_1$,
i.e., there exists a decomposition $P_0 = \gamma P_1 + (1-\gamma)Q'$ with
$\gamma > 0$. Then the decomposition $P_0 = \beta \Ptp + (1-\beta)Q'$
holds with $\beta=\frac{\gamma}{\gamma+(1-\pi_1)(1-\gamma)} \in (0,1]$, so
that $P_0$ is not irreducible with respect to $\Ptp$. Finally, in the case
$\pi_1=1$, we have $\Ptp=P_0$, in which case, trivially, $P_0$ is not
irreducible with respect to $\Ptp$ either.
\end{proof}

The following corollary summarizes the discussion of sufficiency.
\begin{cor}
\label{cor:suff}
If {\bf (A)} and \mutual\, hold, then
$\pi_0 = \frac{\ptn(1-\ptp)}{1-\ptp\ptn}$ and $\pi_1 =
\frac{\ptp(1-\ptn)}{1-\ptp\ptn}$, where $\ptn = \ks(\Ptn|\Ptp)$ and
$\ptp=\ks(\Ptp|\Ptn)$.
\end{cor}
Thus, $\pi_0$ and $\pi_1$ are explicit functions of $\Ptn$ and $\Ptp$
under {\bf (A)} and \mutual\,. It follows that
$P_0$ and $P_1$ can then be
recovered by  solving the identities \eqref{eqn:ssnd0}-\eqref{eqn:ssnd1}.
In fact, using these identities, it is easy to check that a slightly
stronger statement holds: for any arbitrary given $\Ptn \neq \Ptp$,
there is a unique solution $(\pi_0,\pi_1,P_0,P_1)$
of \eqref{eqn:contam0}-\eqref{eqn:contam1} satisfying {\bf (A)} and
\mutual\,. For short, we call this solution
the {\em unique mutually irreducible solution}
of the problem (condition {\bf (A)} being tacitly required.)
The uniqueness and various properties of this particular solution will be
explored in more detail in Theorem~\ref{thm:cplt} below;
in the next Section, we first argue that conditions {\bf (A)} and \mutual\,
are necessary for decontamination in a certain sense.

\subsection{Necessity}
\label{sec:nec}

 As noted earlier,
given $\Ptn\neq \Ptp$, there are in general many solutions
$(\pi_0,\pi_1,P_0,P_1)$ to the equations
\eqref{eqn:contam0}-\eqref{eqn:contam1}, so that
decontamination is not well-defined in the absence of additional conditions.
Requesting mutual irreducibility of $(P_0,P_1)$ is one way to ensure
unicity of the solution, and also has an interpretation in terms of maximum denoising
(see Theorem \ref{thm:cplt} below). But is it in any way a natural assumption?
We now argue that this condition is also the only one ensuring some relatively natural properties of the decontamination operation.

We introduce some additional notation: let $\fP$ denote the set
of probability distributions on $\sX$.
Denote $\fPd$
the set of couples $(P,Q) \in \fP^2$ with $P\neq Q$.
We denote $\psi$ the contamination operator
from $[0,1]^2\times\fP^2$ to $\fP^2$, with
$\psi(\pi_0,\pi_1,P_0,P_1)=(\Ptn,\Ptp)$ given by
\eqref{eqn:contam0}-\eqref{eqn:contam1}.

Let $\phi$ denote a decontamination operator, i.e., a function from
a subset of $\fP^2$ to $[0,1]^2\times\fP^2$ such that
$\phi(\Ptn,\Ptp)$ returns a solution of \eqref{eqn:contam0}-\eqref{eqn:contam1},
in other words $\psi \circ \phi$ is the identity on the domain of $\phi$.
We further denote $\phi:=(\phi_\pi,\phi_P)$,
where $\phi_\pi(\Ptn,\Ptp)$ are the solution contamination weights and
$\phi_P(\Ptn,\Ptp)$ are the solution source distributions.
Finally, given a decontamination operator $\phi$, call the image of
$\phi_P$ the set of {\em $\phi$-sources} -- this is the set of
probability distribution couples that are considered as the uncontaminated
sources by the operator $\phi$ in at least one configuration of observed
contaminated distributions.

\begin{thm}
  \label{thm:necessity}
  Let $\phi$ denote a decontamination operator satisfying the following
  conditions:

(i) Universality: the domain of $\phi$ is $\fPd$;

(ii) Symmetry: if $\phi(\Ptn,\Ptp) = (\pi_0,\pi_1,P_0,P_1)$, then
$\phi(\Ptp,\Ptn) = (\pi_1,\pi_0,P_1,P_0)$\,;


(iii) Continuity of recovered contamination weights: for any fixed $P_0 \neq P_1$, the mapping
\[
(\pi_0,\pi_1) \in
\set{ (\pi_0,\pi_1) \in [0,1]^2 ; \pi_0 + \pi_1 < 1}   \mapsto \phi_\pi(\psi(\pi_0,\pi_1,P_0,P_1))
\]
is continuous\,;

(iv) Stability of recovered sources: for any $\phi$-source $(P_0,P_1)$, there exists
$\eps>0$ such that
for all $\pi_0,\pi_1 \leq \eps$:
\begin{equation}
\label{eq:stab}
\phi_P(\psi(\pi_0,\pi_1,P_0,P_1)) = (P_0,P_1)\,.
\end{equation}
Then $\phi(\Ptn,\Ptp)$ must be the unique mutually
irreducible solution of
\eqref{eqn:contam0}-\eqref{eqn:contam1} for all $\Ptn \neq \Ptp$.
\end{thm}

The interpretation of this result is that mutual
irreducibility is a necessary condition for decontamination
if conditions (i) to (iv) are required. Condition (i)
states that the decontamination operation should be defined on the full
domain of possible (distinct) observed distributions
and can thus be seen as a universality condition.
Condition (ii) is a natural
symmetry requirement. Condition (iii) is a continuity assumption
(changing the mixing weights by an arbitrarily small amount should not
result in a ``jump'' in the returned estimated contamination proportions)
and condition (iv) is a stability condition
(a couple $(P_0,P_1)$ identified as a source should still be output
as a source by the decontamination operator under small enough mutual
mixing proportions.)

{\bf Remark:} Removing one of the ``natural'' requirements (i)-(iv)
invalidates the conclusion. For example, restricting decontamination
to a certain specific model of sources -- say Gaussian distributions --
could give rise to a non-mutually irreducible decontamination, coherent within
that
model but forgoing universality (i). If we remove continuity requirement
(iii), we can find a decontamination operator that is not mutually
irreducible
and satisfies the other conditions by ``tiling'' the solution space: for
any $(P_0,P_1)$ mutually irreducible,
any $(\pi_0,\pi_1)$ such that $\pi_0 + \pi_1 <1$,
$(\Ptn,\Ptp) = \psi(\pi_0,\pi_1,P_0,P_1)$, for $\pi_i \in [\frac{k_i}{n},\frac{k_i+1}{n})$ ($n$ can be chosen
arbitrarily), define the decontamination $\phi$ as
$\phi=(\phi_\pi,\phi_P)$ with
\[\begin{cases}
\phi_\pi(\Ptn,\Ptp) := \left(\frac{\pi_0-k_0/n}{1-(k_0+k_1)/n},
\frac{\pi_1-k_1/n}{1-(k_0+k_1)/n} \right);\\
\phi_P(\Ptn,\Ptp) := \psi\left(\frac{k_0}{n} ,
\frac{k_1}{n}, P_0, P_1\right).
\end{cases}
\]
It is easy to check that $\pi_0 + \pi_1 <1$ implies $\phi_\pi(\Ptn,\Ptp) \in [0,1]^2$ and satisfies {\bf (A)}.
Then the above $\phi$ satisfies (i),(ii) and (iv) but is not the
mutually irreducible
decontamination.
Finally, stability condition (iv) is needed in order to prevent ``trivial''
decontaminations such as $\phi(\Ptn,\Ptp)=(0,0,\Ptn,\Ptp)$, which is obviously
continuous. Excluding the everywhere trivial decontamination is not enough,
as a decontamination could also be trivial on part of the space only.

\subsection{Maximal Denoising}
\label{sec:max}

To conclude this section, we present a result that rounds out the
discussion of the initial and modified contamination models, and mutual
irreducibility. In particular, we describe all possible solutions
$(\pi_0,\pi_1,P_0,P_1)$ to our model
equations~\eqref{eqn:contam0}-\eqref{eqn:contam1} when $\Ptn,\Ptp$ are
given and arbitrary, and an equivalent characterization of the unique
mutually irreducible solution. It can be seen as an analogue of
Proposition \ref{prop:canondecmp} for the label noise contamination
models.

\begin{thm}
\label{thm:cplt}
Let $\Ptp\neq\Ptn$ be two given distinct probability distributions.
Denote by $\Lambda$ the feasible set of quadruples $(\pi_0,\pi_1,P_0,P_1)$ such that
{\bf (A)} and equations \eqref{eqn:contam0}-\eqref{eqn:contam1} are satisfied.
\begin{enumerate}
\item There is a unique quadruple $(\pi_0^*,\pi_1^*,P_0^*,P_1^*) \in \Lambda$
so that \mutual\, holds.
\item Denoting $\ptn^* := \ks(\Ptn|\Ptp)<1$ and $\ptp^* :=
\ks(\Ptp|\Ptn)<1$, it holds
\begin{align}
\label{eqn:expl}
\pi_0^* & = \frac{\ptn^*(1-\ptp^*)}{1-\ptp^*\ptn^*}, &  \pi_1^* = \frac{\ptp^*(1-\ptn^*)}{1-\ptp^*\ptn^*}\,.
\end{align}
\item The feasible region $R$ for the proportions $(\pi_0,\pi_1)$ (that is, the projection of $\Lambda$
to its first two coordinates, which is also one-to-one),
is the closed quadrilateral defined by the intersection of the positive
quadrant
of $\mathbb{R}^2$ with the half-planes given by
\begin{align}
\label{eqn:feas}
\pi_0 + \pi_1 \ptn^* & \leq \ptn^*, & \pi_1 + \pi_0 \ptp^* & \leq \ptp^*\,.
\end{align}
\item The mutually irreducible solution $(\pi_0^*,\pi_1^*,P_0^*,P_1^*)$ is
also equivalently characterized as:
\begin{itemize}
\item the unique maximizer of $(\pi_0+\pi_1)$ over $\Lambda$;
\item the unique extremal point of $\Lambda$ where both of the constraints in \eqref{eqn:feas} are active;
\item the unique maximizer over $\Lambda$ of the total variation distance $\norm{P_0-P_1}_{TV}$.
\end{itemize}
\end{enumerate}
\end{thm}
The proof of the theorem relies on the explicit one-to-one correspondence
established in Lemmas~\ref{le:le1} and~\ref{le:le1conv} between the solutions of the original decomposition
\eqref{eqn:contam0}-\eqref{eqn:contam1} and its decoupled reformulation \eqref{eqn:ssnd0}-\eqref{eqn:ssnd1}.
The result of Proposition~\ref{prop:canondecmp}
is applied to the decoupled formulation, then pulled back, via the correspondence,
in the original representation. The last statement concerning the total variation norm is based on the
relation
\[
(P_1-P_0) = (1 - \pn- \pp)^{-1} (\Ptp - \Ptn),
\]
obtained by subtracting \eqref{eqn:contam0} from \eqref{eqn:contam1}. Therefore, the maximum feasible value of
$\norm{P_1 - P_0}_{TV}$ corresponds to the maximum of $(\pi_0+\pi_1)$,
i.e., the unique mutually irreducible solution.

The geometrical interpretation of this theorem is visualized on
Figure~\ref{fig:geo}. In particular, point 1 of the
theorem shows that
conditions {\bf (A)} and \mutual\, do not restrict the class of possible
observable contaminated distributions $(\Ptp,\Ptn)$; rather, they ensure
in all cases the identifiability of the mixture model.
Point 4 indicates that the unique
solution satisfying the mutual irreducibility condition \mutual\, can be
characterized as maximizing the possible total label noise
level $(\pi_0+\pi_1)$, or, still equivalently, the total variation separation of the source
probabilities $P_0,P_1$. In this sense, the mutually irreducible solution
can also be interpreted as {\em maximal label denoising} or {\em maximal source
separation} of the observed contaminated distributions.

\begin{figure}
\begin{center}
\scalebox{0.33}{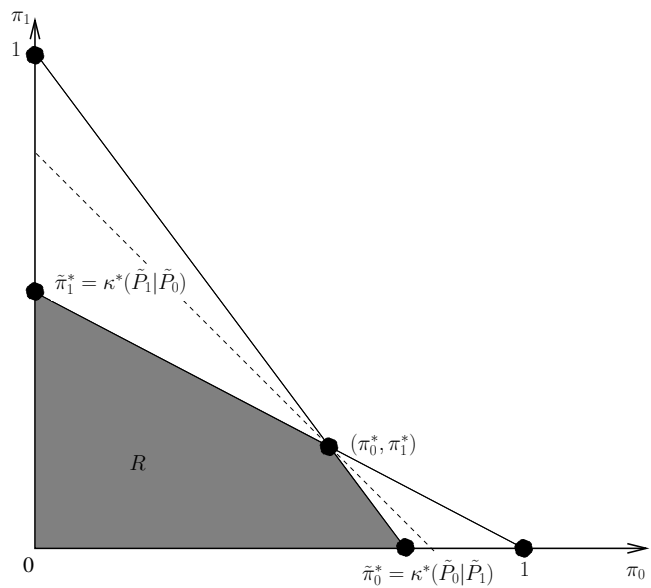}
\end{center}
\caption{\label{fig:geo} Geometry of the feasible region $\Lambda$ for
proportions $(\pi_0,\pi_1)$ solutions of
the contamination model \eqref{eqn:contam0}-\eqref{eqn:contam1}, when contaminated
distributions $(\Ptn,\Ptp)$ are observed and the true distributions $(P_0,P_1)$ are unknown.
Each feasible $(\pi_0,\pi_1)$ corresponds to a single associated solution $(P_0,P_1)$.
The extremal point $(\pi_0^*,\pi_1^*)$ is the unique
point corresponding to a mutually irreducible solution $(P_0^*,P_1^*)$. The
dashed line indicates the maximal level line $(\pi_0+\pi_1)=c$ intersecting with $\Lambda$.}
\end{figure}

\section{Mixture Proportion Estimation and a Rate of Convergence}
\label{sec:mpe}

\citet{blanchard10} present a universally consistent estimator $\khat$ of
$\ks(F|H)$. We review this estimator below. They also establish a ``no
free lunch" result stating that no estimator of $\ks(F|H)$ can converge at
a fixed rate for all $F$ and $H$. In this section we also introduce
distributional assumptions under which the estimator of
\citet{blanchard10} converges at a known rate.

We begin by reviewing the universally consistent estimator of $\ks(F|H)$ introduced by \citet{blanchard10}.
Let $F$ and $H$ be probability measures on a Borel space $(\sX, \fS)$. Recall from Proposition \ref{prop:ks}
$$
\kappa^{*}(F|H)=\inf_{S \in \fS, H(S) > 0} \, \frac{F(S)}{H(S)}.
$$
The basic idea is to replace $F$ and $H$ by empirical estimates and take
the infimum over a union of VC classes. Thus,
consider a sequence of VC classes of sets,
$(\sS_k)_{k\geq 1}$, with respective
(finite) VC dimensions $(V_k)_{k\geq 1}$. Define
$\eps_i(k,\delta_i) := 3 \sqrt{\frac{V_k \log (n_i+1) - \log
\delta_i/2}{n_i}}$\, for $i=0,1$.
By the VC inequality, for any $i=0,1$, $\delta_i \in (0,1)$,
$k\ge 1$ and any distribution $Q$ on $\sX$\,,
with probability at least $1-\delta_i$\, over the draw of an i.i.d. sample
of size $n_i$ according to $Q$\,, we have
\begin{equation}
\label{unibound}
\forall S \in \sS_k\, \ \ \abs{Q(S) - \wh{Q}(S)} \leq
\eps_i(k,\delta_i)\,,
\end{equation}
where $\wh{Q}$ denotes the empirical distribution built on the sample.

In MPE we have training data
\begin{align}
X_0^1, \ldots, X_0^{n_0} & \stackrel{iid}{\sim} H, \\
X_1^1, \ldots, X_1^{n_1} & \stackrel{iid}{\sim} F.
\end{align}
For $k \ge 1$, define
\begin{equation}
\label{eqn:khat3}
\khat(k, \delta_0, \delta_1) := \inf_{S \in \sS_k} \frac{\wh{F}(S) +
\eps_1(k,\delta_1)}{(\wh{H}(S) - \eps_0(k,\delta_0))_+}
\end{equation}
where $( \cdot )_+$ is the max of its argument and zero (the ratio is defined
to be $\infty$ if the denominator is zero), and where
$\wh{F}(S)$ and $\wh{H}(S)$ are the empirical true positive
and false positive probabilities associated with the rejection region $S$. By
the VC inequality and Proposition \ref{prop:ks},
$\khat(k, \delta_0, \delta_1)$ is an upper bound on $\ks(F|H)$, with
probability at least $1-\delta_0 - \delta_1$.


Next, define
$$
\khat(\delta_0,\delta_1) := \inf_{k \ge 1} \, \khat(k,
\delta_0 k^{-2}, \delta_1 k^{-2}).
$$
By the union bound, this is also an upper bound on $\ks$, with probability at
least $1-2(\delta_0 + \delta_1)$, since $\sum_k k^{-2} = \pi^2/6 < 2$.
To ensure that this upper
bound approaches $\ks$ as $n_0, n_1 \to \infty$, the sequence
$(\sS_k)_{k=1}^\infty$ is assumed
to satisfy the following universal approximation property, which we refer to as
{\bf (AP1)}:
For any $S^* \in \fS$\,, and any
distribution $Q$\,,
\begin{equation*}
\liminf_{k \rightarrow \infty} \inf_{S \in \sS_k} Q(S \Delta S^*)
=0\,,
\end{equation*}
where $S \Delta S^* = S \backslash S^* \cup S^* \backslash S$ is the symmetric
set difference.

Finally, $\khat$ is defined as $\khat = \khat(\frac1{n_0},\frac1{n_1})$.
\citet{blanchard10} show the following, which makes no assumption on the
distributions $F$ and $H$ and thus establishes a universally consistent method
for MPE.
\begin{thm}[\citet{blanchard10}]
\label{thm:consist}
With probability at least $1 - 2(\frac1{n_0}+\frac1{n_1})$, $\khat \ge \ks(F|H)$.
Furthermore, if $(\sS_k)_{k=1}^\infty$ satisfies {\bf (AP1)}, then $\khat
\stackrel{\mbox{i.p.}}{\longrightarrow} \ks(F|H)$ as $\min\{n_0,n_1\} \to
\infty$.
\end{thm}
It should be noted that the statement of the consistency result of
\citet{blanchard10} contains a slight error. We present a correction to
the statement of the consistency result in an appendix; see also
\citet{scottWSLnotes}. The error/correction does not affect the present
work.

We now introduce an assumption on $F$ and $H$ that will ensure a certain
rate of convergence for $\khat$ above. This rate will be used in the next
section to establish consistency of a discrimination rule. The support of
a distribution $Q$, denoted $\supp(Q)$, is the smallest closed set whose
complement has measure zero.
\begin{description}
\item[(D)] There exists a distribution $G$ and $\gamma \in [0,1]$ such
that
$\supp(H) \not\subset \supp(G)$ and $F = (1-\gamma) G + \gamma H$.
\end{description}
The assumption $\supp(H) \not\subset \supp(G)$ clearly implies that $G$ is
irreducible with respect to $H$, and therefore $\gamma$ in {\bf (D)} is
equal
to $\ks(F|H)$.

In addition, we adopt a modified approximation condition on the sequence
$(\sS_k)$, referred to as {\bf (AP2)}: For all $G$, $H$ with $\supp(H) \not
\subset \supp(G)$ there exists $k \ge 1$ and $S \in \sS_k$ s.t. $G(S) = 0$ and
$H(S) > 0$.

{\bf Remark:} {\bf (AP1)} requires that the sets in $\sS_k$ become
increasingly
complex, so that $V_k \to \infty$. On the other hand, {\bf (AP2)} does not. For
example, if $\sX = \reals^d$ and $\fS$ is the Borel $\sigma$-algebra
generated by the standard topology on $\reals^d$, {\bf
(AP2)} is satisfied taking $\sS_1$ to be the VC class of all open balls $\{x
: \| x - c \| < r\}, c \in \reals^d, r > 0$, and $\sS_k = \emptyset$ for
$k \ge 2$. In this
case,
we could even simplify the estimator of $\ks$ to be $\khat' :=
\khat(1,\frac1{n_0},\frac1{n_1})$, and the rate of convergence presented below
would still hold (the proof requires only minor modifications). However, we elect
to work with the definition of $\khat$ above to emphasize that the rate of
convergence applies to the universally consistent estimator.
\begin{thm}
\label{thm:mperate}
Suppose $(\sS_k)_{k\ge 1}$ is chosen to satisfy {\bf (AP2)}.
If $F$ and $H$ are such that {\em \bf (D)} holds, then there exists a
constant $C>0$
such that for $n_0$ and $n_1$ sufficiently large, the estimator $\khat$ satisfies
\begin{equation}
\label{eqn:mperate}
\Pr\Bigg(\left|\khat - \ks\right| \ge C\Bigg[\sqrt{\frac{\log n_0}{n_0}} +
\sqrt{\frac{\log n_1}{n_1}} \Bigg] \Bigg)
\le \frac2{n_0} + \frac2{n_1}
\end{equation}
where $\ks = \ks(F|H)$.
\end{thm}
In the next section, assume $\khat$ is defined in terms of VC classes
satisfying {\bf
(AP2)}.

\section{Consistent Classification with Unknown Label Noise Proportions}
\label{sec:consist}

The consistent estimator of $\ks$ just discussed provides a clear path to
the design of a consistent discrimination rule when the label noise
proportions are unknown. The estimator of $\ks$, together with Corollary
\ref{cor:suff}, can be combined to give consistent estimators of $\ptn$
and $\ptp$ under assumptions {\bf (A)} and {\bf (C)}. Plugging in these
estimators, along with empirical estimates of $\Rtn$ and $\Rtp$, into
Eqns. \eqref{eq:noisyz2} and \eqref{eq:noisyo2}, yields estimates of $R_0$
and $R_1$
that can
be shown to converge uniformly over a VC class of classifiers to their
true values. By allowing the size of the VC class to grow as the sample
size(s) grow, empirical risk minimization can be shown to be a consistent
discrimination rule with respect to any performance measure defined in
terms of $R_0$ and $R_1$. This idea utilizes standard ideas in learning
theory and is illustrated for the minmax criterion in \citet{scott13colt}.

One drawback of empirical risk minimization over VC classes is that it is
computationally intractable for most VC classes of interest. In the
remainder of this section we establish a computationally tractable
consistent discrimination rule based on surrogate risk minimization.

\subsection{Problem Formulation}

Let
$(X,Y)$ be random on $\sX \times \{0,1\}$ where $\sX$ is a Borel
space, and let $P$ denote the probability measure governing $(X,Y)$. Let $\sM$
denote the set of decision functions, i.e., the set of measurable functions $\sX
\to \reals$.  Every $f \in \sM$ induces a classifier $x \mapsto u(f(x))$
where $\step(t)$ is the unit step function
$$
\step(t) := \bcase 1, & t > 0 \\ 0, & t \le 0. \ecase
$$
For any
$f \in \sM$, define the {\em cost-insensitive P-risk} of $f$
$$
R_P(f) := \ev_{(X,Y)\sim P}[\ind{\step(f(X)) \ne Y}].
$$
Define the {\em cost-insensitive Bayes $P$-risk} $R_P^* := \inf_{f \in \sM}
R_P(f)$. It is
well known \citep{devroye96} that for any $f \in \sM$, the excess $P$-risk
satisfies
\begin{equation}
\label{eqn:excess}
R_P(f) - R_P^* = 2 \ev_X[\ind{\step(f(X)) \ne
\step(\eta(X)-\frac12)}|\eta(X)-\half|],
\end{equation}
where $\eta(x) := P(Y = 1 \, | \, X = x)$.

Generalizing the above, for any $\a \in (0,1)$ we can define the
{\em $\a$-cost-sensitive $P$-risk} for any $f \in \sM$,
\begin{align*}
R_{P,\a}(f) := \ev_{(X,Y)\sim P}[&(1-\a)\ind{Y=1}\ind{f(X)\le0} \\
& \ \ \ \ \ \ +\a\ind{Y=0}\ind{f(X)>0}].
\end{align*}
The corresponding Bayes risk is $R_{P,\a}^* := \inf_{f \in \sM}
R_{P,\a}(f)$, and the analogue to \eqref{eqn:excess} is \citep{scott12}:
\begin{equation}
\label{eqn:excesscost}
R_\a(f)-R_\a^* = \ev_X[\ind{\step(f(X)) \ne \step(\eta(X)-\a)}|\eta(X)-\a|]\,.
\end{equation}
Observe \eqref{eqn:excess}
corresponds to the case $\a = \frac12$.

With this background, we turn to the problem of classification with label
noise. We assume $(X,Y,\Yt)$ are jointly distributed, where $Y$ is the
true but unobserved label, and $\Yt$ is the observed but noisy label. As
in the rest of the paper, we focus on label noise that is independent of
the feature vector $X$, meaning that the conditional distribution of $\Yt$
given $X$ and $Y$ depends only on $Y$.

We would like to minimize $R_P(f)$, but we only have access to data from
$\Pt$, the joint distribution of $(X,\Yt)$. \citet{tewari13} show that
minimizing a cost-{\em sensitive} $\Pt$-risk is equivalent to minimizing
the cost-{\em insensitive} $P$-risk. We state and prove an equivalent
result which has a simpler proof. In this setting, $\pi_i = \Pr(Y=1-i \, |
\, \Yt = i)$, $i=0,1$. We introduce the following assumption on the amount
of label noise, which slightly strengthens {\bf (A)}.
\begin{description}
\item[(A')] $\pn < \frac12$ and $\pp < \frac12$.
\end{description}
The following result connects the cost-sensitive $\Pt$-risk to the
cost-insensitive $P$-risk.
\begin{lemma}
\label{lem:excess}
If {\bf (A')} holds, then for any $f \in \sM$,
\begin{equation}
R_{P}(f) - R_{P}^* = 2 (1 - \pp - \pn) (R_{\Pt,\a}(f) - R_{\Pt,\a}^*)\,,
\end{equation}
where $ \a = (\frac12 - \pn)/(1 - \pp - \pn)$.
\end{lemma}
\begin{proof}
Note that {\bf (A')} ensures $\a \in (0,1)$.
Define $\etat(x)$ in analogy to $\eta(x)$ by $\etat(x):= \Pr(\Yt = 1 |
X=x)$, leading to
\begin{align*}
  \eta(x) &= \Pr(Y=1,\Yt=1 | X=x)
  + \Pr(Y=1,\Yt=0 | X=x) \\
  &= \Pr(Y=1|\Yt=1,X=x)\etat(x)
  + \Pr(Y=1|\Yt=0,X=x)(1-\etat(x)) \\
&= (1-\pi_1) \etat(x) + \pi_0 (1-\etat(x)) \\
&= (1-\pi_0 - \pi_1) \etat(x) + \pi_0.
\end{align*}
Observe that
\begin{align*}
\eta(x) - \half &= (1-\pi_0 - \pi_1) \etat(x) + \pi_0 - \half \\
&= (1-\pi_0 - \pi_1) [\etat(x) - \alpha].
\end{align*}
The result follows now from \eqref{eqn:excess} and \eqref{eqn:excesscost}:
\begin{align*}
&R_P(f) - R_P^* = 2\ev_X \Big[ \ind{u(f(X)) \ne
u(\eta(x)-\half)}|\eta(x)-\half| \Big] \\
&= \ \ 2 (1-\pi_0 - \pi_1) \ev_X \Big[ \ind{u(f(X)) \ne
u(\etat(x)-\a)}|\etat(x)-\a| \Big] \\
&= \ \ 2 (1 - \pp - \pn) (R_{\Pt,\a}(f) - R_{\Pt,\a}^*).
\end{align*}
\end{proof}

The problem we will address is the construction of a discrimination rule
$\fhat$ that is computationally tractable, does not know $\a, \pi_0$, or
$\pi_1$, and is such that $R_{P}(\fhat) - R_{P}^* \to 0$ in probability.
To achieve this, we develop an algorithm $\fhat$ based on surrogate risk
minimization such that $R_{\Pt,\a}(\fhat) - R_{\Pt,\a}^* \to 0$ in
probability.

\subsection{Surrogate Losses}
\label{sec:loss}

A {\em loss} is any measurable function $L: \{0,1\} \times \reals \to
[0,\infty)$. For example, the $P$-risk is defined in terms of the $0-1$
loss, $L(y,t) = \ind{y \ne \step(t)}$. Given a loss $L$ we define the risk
$$
R_{P,L}(f) = \ev_{(X,Y)\sim P}[L(Y,f(X))],
$$
and the corresponding optimal risk $R_{P,L}^* = \inf_{f \in \sM}
R_{P,L}(f)$.

A {\em surrogate loss} is one that is used as a surrogate for another,
such as a loss $L$ that is convex in its second argument in lieu of the
0-1 loss. Surrogate losses are common in machine learning because they can often
be optimized efficiently, unlike the 0-1 loss and its cost-sensitive variants.
The notion of classification calibration was developed to
theoretically justify the use of surrogate losses. A loss $L$ is said to be
{\em $\a$-classification calibrated} iff there exists an
increasing and continuous function $\theta$ with $\theta(0) = 0$ such that
for all $f \in \sM$,
$$
R_{P,\a}(f) - R_{P,\a}^* \le \theta(R_{P,L}(f) - R_{P,L}^*).
$$
An equivalent and more technical characterization of $\a$-CC is provided
by \cite{scott12}, but the above definition suffices for our purposes. The point
is
that driving the surrogate excess risk to zero drives the target excess
risk to zero for $\a$-CC losses, and the former can be accomplished by
computationally tractable methods like support vector machines, as shown below.

Any loss $L$ can be expressed as $L(y,t) = \ind{y = 1} L_1(t) + \ind{y = 0}
L_0(t)$. Given a loss $L$ and $\a \in (0,1)$, define
\begin{equation}
\label{eqn:alphaloss}
L_{\a}(y,t) := (1-\a) \ind{y = 1} L_1(t) + \a \ind{y = 0} L_0(t).
\end{equation}
\cite{scott12} establishes that $L$ is $\frac12$-CC iff $L_\a$ is $\a$-CC.
Several examples of $\frac12$-CC losses are known, so these readily translate to
examples of $\a$-CC losses via Eqn. \eqref{eqn:alphaloss}. In particular,
\citet{bartlett06} establish that if $L(y,t)=\phi((2y-1)t)$ where $\phi$
is convex and
differentiable at 0 with $\phi'(0) < 0$, then $L$ is $\frac12$-CC. This justifies
several common losses including the hinge loss ($\phi(z) = \max\{0,1-z\}$) and
the logistic loss ($\phi(z) = \log(1+\exp(-z))$).
Combining
these ideas with Lemma \ref{lem:excess} leads to the following result.
\begin{cor}
\label{prop:regret}
Suppose $L$ is $\half$-CC, assume {\bf (B)} is
satisfied and let $\a = (\frac12 - \pn)/(1 - \pp - \pn)$. Then there exists an
increasing and continuous function $\theta$ with
$\theta(0) = 0$ such that
for all $f
\in \sM$,
$$
R_{P}(f) - R_{P}^* \le \theta(R_{\Pt,L_\a}(f) - R_{\Pt,L_\a}^*).
$$
\end{cor}
\citet{tewari13} consider the setting where $\pn$ and $\pp$ are known. Using the
above result, they apply Rademacher complexity analysis
to establish performance guarantees for a classification strategy based on
regularized
empirical risk minimization with a surrogate loss $L_\a$.

\subsection{Estimating $\a$}
\label{sec:alphaest}

When $\pn$ and $\pp$ are unknown, a natural strategy is to base a learning
algorithm on a surrogate loss $L_{\wa}$, where $\wa$ is an estimate of $\a$. We
propose an estimate of the form
$$
\wa = \frac{\half - \pnhat}{1 - \pnhat - \pphat},
$$
where $\pnhat$ and $\pphat$ are estimates based on our previously
developed results. In particular, suppose we observe noisy data
$$
(X_1, \Yt_1), \ldots, (X_n, \Yt_n) \stackrel{iid}{\sim} \Pt,
$$
One difference to note going forward is that the sample sizes $n_0$ and
$n_1$ are now random, whereas before they were considered to be nonrandom.
This turns out to be a minor difference; see the proof of Proposition
\ref{prop:arate} below.


Now, let $\ptnhat$ and $\ptphat$ be estimates of $\ptn$ and
$\ptp$ obtained by applying the estimator $\khat$ of Section
\ref{sec:mpe} twice. The formulas from Lemma \ref{le:le1conv} lead
to the estimates
\begin{equation}
\label{eqn:piest}
\pnhat = \frac{\ptnhat(1-\ptphat)}{1-\ptnhat \ptphat}
\text{ \ \ \ \ \ and \ \ \ \ \ }
\pphat = \frac{\ptphat(1-\ptnhat)}{1-\ptnhat \ptphat}.
\end{equation}
By Corollary \ref{cor:suff}, if {\bf (A)} and {\bf (C)} hold, then
$\ptn = \ks(\Ptn | \Ptp)$ and $\ptp = \ks(\Ptp | \Ptn)$, and consequently
$\pnhat$ and $\pphat$ are consistent estimators of $\pi_0$ and $\pi_1$,
respectively. For some of our subsequent analysis, we actually want
$\pnhat$ and
$\pphat$ (and therefore $\wa$) to converge at a known rate. Hence,
we want $\Ptn$ and
$\Ptp$ to satisfy assumption {\bf (D)} in both directions.
The following assumption, which strengthens {\bf (C)}, is sufficient for
this purpose.
\begin{description}
\item[(C')] $\supp(P_0) \not \subset \supp(P_1)$ and $\supp(P_1) \not
\subset
\supp(P_0)$.
\end{description}
This assumption is reasonable in many classification problems. It essentially
says that for each of
the two (noise-free) classes, there exist patterns belonging to that class that
could not possibly be confused with patterns from the other class. We have the
following.
\begin{prop}
\label{prop:arate}
If {\bf (A')} and {\bf (C')} hold, then there exist $C_1, C_2 > 0$ such
that for $n$ sufficiently large,
$$
\Pr \left( |\wa - \a| \ge C_1 \sqrt{\frac{\log n}{n}} \right) \le
\frac{C_2}{n}.
$$
\end{prop}
\begin{proof}
{\bf (A')} implies $\pn + \pp < 1$, and by {\bf (C')}, $P_0$ and $P_1$ are
mutually irreducible. Thus
Corollary \ref{cor:suff} implies $\ptn = \ks(\Ptn|\Ptp)$ and $\ptp =
\ks(\Ptp|\Ptn)$. We will apply Theorem \ref{thm:mperate} to both of the
estimators
$\ptnhat$ and $\ptphat$. To verify the assumptions of that theorem, we need to
verify {\bf (D)} for both $(F,H) = (\Ptp,\Ptn)$ and  $(F,H) =
(\Ptn,\Ptp)$. We will show {\bf (D)} for $(F,H) = (\Ptp,\Ptn)$, the other
case
being similar. From \eqref{eqn:ssnd1}, it suffices to show $\supp(\Ptp)
\not\subset \supp(P_0)$. But this holds because $\Ptp = (1-\pp)P_1 + \pp P_0$
(see Eqn. \eqref{eqn:contam1}) and $\supp(P_1) \not \subset \supp(P_0)$ and $\pp
< 1$. We can now apply Theorem
\ref{thm:mperate} to both $\ptnhat$ and $\ptphat$. To do so, since $n_0$
and $n_1$ are nonrandom in that result, we must condition on $n_0$ and
$n_1$, and appeal to the fact that, with high probability, $n_0$ and $n_1$
are proportional to $n$. In particular, if $\qt = \Pr(\Yt = 1)$, then the
relative Chernoff bound implies that with high probability, $n_1 \in
(\frac12 \qt n, \frac32 \qt n)$ and $n_0 \in
(\frac12 (1-\qt) n, \frac32 (1-\qt) n)$. Conditioning on $n_0$ and $n_1$
belonging to these intervals, Theorem \ref{thm:mperate} implies that both
$\ptphat$ and $\ptnhat$ converge at rates that are $O(\sqrt{\log n/n})$.
These rates lead to similar rates for $\pphat$
and $\pnhat$ (note in particular that assumption {\bf(A')} implies that
the denominators in \eqref{eqn:piest} are bounded away
from 0 by a fixed margin with large probability for $n$ large enough,
independently of $\pp,\pn$).
This in turn
leads to the desired rate for $\wa$.
\end{proof}

\subsection{Algorithm}


We now introduce a consistent classification procedure based on surrogate losses
in the case of unknown label noise proportions.
The algorithm relies on the framework of reproducing kernel Hilbert spaces. Thus,
let $\sH$ be a RKHS, and let $L$
be a loss for binary classification. We say that $L$ is Lipschitz if $L(y,t)$ is
a Lipschitz function of $t$ for each $y$. The algorithm returns the classifier
\begin{equation}
\label{eqn:alg}
\fhat = \argmin_{f \in \sH} \frac1{n} \sum_{i=1}^n L_{\wa}(\Yt_i, f(X_i))
+
\lambda_n \| f \|_{\sH}^2,
\end{equation}
where $ L_{\wa}$ is the $\wa$-weighted cost-sensitive loss associated with $L$,
as defined in \eqref{eqn:alphaloss}. For example, if $L(y,t) =
\max\{0,1-(2y-1)t\}$ is
the hinge loss, $\fhat$ is a cost-sensitive support vector machine.

\subsection{First Consistency Result}

We will assume that the reproducing kernel $k$ associated with
$\sH$ is universal and bounded \citep{steinwart08}. The former property
implies
that elements of the RKHS can get arbitrarily close to the Bayes risk. The
latter property states that $\sup_x k(x,x) =: B^2 < \infty$. The Gaussian kernel
is
an example satisfying both of these properties.
\begin{thm}
\label{thm:discrim}
Assume {\bf (A')} and {\bf (C')} hold, that the reproducing kernel
associated with
$\sH$ is universal and bounded, and that $L$ is a
Lipschitz, $\half$-CC loss. Let
$\lambda_n > 0$ tend to zero as $n \to \infty$ such that $\lambda_n \sqrt{n /
\log n} \to \infty$. Then
$$
R_P(\fhat) - R_P^* \to 0\, \ \ \ \ \mbox{ in probability\,,}
$$
as $n \to \infty$.
\end{thm}

\subsection{Alternate Consistency Result with Clippable Losses}

It is possible to establish a consistency
theorem without requiring a rate of convergence on
$\wa$\, (thus only requiring the milder condition {\bf (C)} rather than
{\bf (C')}), at the expense of treating a more narrow class of losses.

A {\em $T$-clippable loss} $L(y,t)$
(see \citealp{steinwart08}, Section 2.2) satisfies the following
property:
\[
\forall y \in \set{0,1}\,, \forall t \in \reals\;\; : L(y,\clip_T(t)) \leq L(y,t)\,,
\]
where $\clip_T(t) := \min(T,\max(-T,t))\,.$
It is shown by \cite{steinwart08}, Lemma 2.23, that a convex loss
is $T$-clippable iff $\forall y \in \set{0,1}$,
the function $t \in \reals \mapsto L(y,t)$
admits a minimum which is attained for some $t\in[-T,T]$.
As a consequence, many common surrogate losses are clippable;
for instance the hinge loss, the squared loss and
the truncated squared loss are 1-clippable. On the other hand,
the logistic and the exponential losses are not clippable.

\begin{thm}
  \label{thm:clippable}
  Assume {\bf (A')} and {\bf (C)} hold, that the reproducing kernel
associated with
$\sH$ is universal and bounded, and that $L$ is a
Lipschitz, $T$-clippable, $\half$-CC loss. Let
$\lambda_n > 0$ tend to zero as $n \to \infty$ such that $\lambda_n n \to
\infty$. Define $\cfhat := \clip_T(\fhat)$\,,
where $\fhat$ is defined by \eqref{eqn:alg}.
Then
$$
R_P(\cfhat) - R_P^* \to 0 \ \ \ \ \mbox{ in probability,}
$$
as $n \to \infty$.
\end{thm}

\section{A More General Analysis of Co-Training}
\label{sec:cotrain}

Co-training is a model for binary classification in which the feature
vector can be partitioned into two sets of variables, called ``views." The
critical assumption of co-training is that the views are conditionally
independent, given the class label. We refer to this assumption as the
{\em co-training assumption}. \citet{blum98} show that under this
assumption, the optimal classifier can be learned from {\em unlabeled
data} only, provided the learner has access to a ``weakly-useful
predictor," which is a classifier that, roughly speaking, is at least
slightly better than random guessing. The basic idea is to apply the
weakly-useful predictor to one of the views to generate noisy labels for
the other view. By the co-training assumption, the problem now reduces to
classification with label noise. The original analysis assumes that the
true label is a deterministic function of either view. Our framework
allows us to relax this assumption.

To state our result, we assume that the feature vector $X$ and label $Y$
are jointly distributed with joint distribution $Q$. Let $P_0$ and $P_1$
be the class conditional distributions of $Q$. Furthermore, let $X$ be
expressed as $(X^A, X^B)$, representing the two views. Under the
co-training
assumption, $X^A$ and $X^B$ are conditionally independent given $Y$. The
unlabeled training data are $X_1, \ldots, X_n$. A {\em weakly-useful
classifier}
is a classifier $h$ such that $0 < Q(\{x : h(x) = 1\}) < 1$ and  $q_0(h)
+ q_1(h) < 1$, where
$$
q_i(h) = Q(Y = 1-i \, | \, h(X) = i).
$$

\begin{thm}
Let $h^A$ be a known weakly-useful classifier based on view $A$. Assume
that the class-conditional distributions of $X^B$ are mutually
irreducible, and let $X_1, \ldots, X_n$ be iid. Under the
co-training assumption, there exists a classification algorithm $\fhat$
such that $R_Q(\fhat) \to R_Q^*$ in probability as $n \to \infty$.
\end{thm}

\begin{proof}
Consider the data set
$$
(X^B_1, \tilde{Y}_1), \ldots,  (X^B_n, \tilde{Y}_n),
$$
where $\tilde{Y}_i = h^A(X^A)$. By the co-training assumption, the
class-conditional distribution of $\tilde{Y}$ given
$X^B$ and the true label $Y$ is not dependent on $X^B$. Therefore we have
the setting of a label noise problem. Since $0 < Q(\{x^A : h^A(x^A) =
1\}) < 1$, the numbers of examples $n_0$ and $n_1$ with each noisy label
grow with $n$. Furthermore, the contamination probabilities
$$
\pi_i = \Pr(Y = 1-i \, | \, \tilde{Y}=i)
$$
are just $\pi_i = q_i(h^A)$. Since $h^A$ is weakly-useful, we have that
$\pi_0 + \pi_1 < 1$. We also have mutual irreducibility for this label
noise problem, by assumption. Therefore, a consistent classification rule
exists by the construction in \citet{scott13colt}.
\end{proof}

The key point is that this result weakens the assumption of deterministic
class labels to a mutual irreducibility assumption. The existence of a
weakly-useful classifier could be guaranteed, for example, if a small
amount of labeled training data was available.

The previous argument relies on the consistent classification rule from
\cite{scott13colt}. The consistency result for classifiers based on clippable
surrogate losses, from earlier in this paper, could also be employed provided
the definition of a weakly-useful classifier is strengthened to require that
$q_i(h) < \frac12$ for each $i$.

\section{Mutual Irreducibility and Class Probability Estimation}
\label{sec:cpe}

In this section, we relate mutual irreducibility of $P_0$ and $P_1$ to
the problem of class probability estimation. Let $p_0$ and $p_1$ be
densities of $P_0$ and
$P_1$ with respect to a common dominating measure.
Further assume that the feature vector $X$ and label $Y$ are jointly
distributed with joint distribution $P$, and that $q := P(Y=1) \in (0,1)$.
The posterior probability that $Y=1$ is denoted
$$
\eta(x) := P(Y=1 \, | \, X=x ).
$$
The problem of estimating $\eta$ from data is known as class probability
estimation \citep{buja05,reid10}. The most well-known approach to class
probability estimation is logistic regression, which posits the model
$$
\widehat{\eta}(x) = \frac1{1 + \exp\{-(w^T x + b)\}},
$$
where $w$ and $x$ have the same dimension, and $b \in \R$. The
parameters $w$ and $b$ are fit to the data by maximum likelihood. More
generally, estimates for $\eta$ commonly have the form
$$
\widehat{\eta}(x) = \psi^{-1}(h(x))
$$
where $\psi: [0,1] \mapsto \R$ is a 
{\em link} function,
and $h$ is a decision function of some sort.


Now define
$$
\emin := \essinf_{x \in \sX} \ \eta(x) \ \ \ \ \ \mbox{ and } \ \ \ \ \
\emax := \esssup_{x \in \sX} \ \eta(x).
$$
The following result connects the posterior class probability to mutual
irreducibility.

\begin{prop}
\label{prop:cpe}
With the notation defined above,
\begin{equation}
\label{eqn:emax}
\emax = \frac{1}{1+\frac{1-q}{q} \ks(P_1|P_0)}
\end{equation}
and
\begin{equation}
\label{eqn:emin}
\emin = 1 - \frac{1}{1+\frac{q}{1-q} \ks(P_0|P_1)}.
\end{equation}
Therefore, $P_0$ and $P_1$ are mutually irreducible if and only if $\emin
= 0$ and $\emax = 1$.
\end{prop}
\begin{proof}
By Bayes' rule, it is true that almost everywhere,
\begin{align*}
\eta(x) &= \frac{q p_1(x)}{q p_1(x) + (1-q) p_0(x)} \\
&= \frac1{1 + \frac{1-q}{q} \frac{p_0(x)}{p_1(x)}}.
\end{align*}
Equation \eqref{eqn:emax} now follows from Proposition \ref{prop:ks}.
Similarly, we have (almost everywhere)
\begin{align*}
\eta(x) &= 1 - \frac{(1-q) p_0(x)}{(1-q) p_0(x) + q p_1(x)} \\
&= 1 - \frac1{1 + \frac{q}{1-q} \frac{p_1(x)}{p_0(x)}}.
\end{align*}
Now \eqref{eqn:emin} follows from Proposition \ref{prop:ks}. The final
statement follows from \eqref{eqn:emax} and \eqref{eqn:emin} and the
definition of mutual irreducibility.
\end{proof}
Thus, estimates of $\ks(P_0|P_1)$ and $\ks(P_1|P_0)$ could be used to
inform choices about the design of the link function (e.g., its domain) and
model class of decision functions.


Proposition \ref{prop:cpe} also suggest another
possible approach to mixture proportion estimation. Suppose
$\widehat{\eta}$ is an estimator for $\eta$ that is consistent with
respect to the supremum norm, and let $\widehat{q}$ be the empirical
estimate of $q$ based on a random sample from $P$. Inverting Equation
\eqref{eqn:emax},
$$
\widehat{\kappa}_{1,0} := \left(\frac1{\sup_{x \in \sX}
\widehat{\eta}(x)}-
1\right) \frac{\widehat{q}}{1-\widehat{q}},
$$
is a consistent estimate of $\ks(P_1|P_0)$. Similar remarks apply to
$\ks(P_0|P_1)$. Although this suggests that class probability estimation
solves mixture proportion estimation in the binary classification context, we
note that sup-norm consistency will require distributional assumptions, and
therefore the distribution-free estimator of \citet{blanchard10} is a more
general solution.

All of the above observations were present in our original technical
report on this topic \citep{scott13tr}. Since then, \citet{liutao16} and
\cite{menon15icml} have further explored the idea of estimating label
noise
proportions from the minimum and maximum of the contaminated class
probability function. In particular, we note the following.

An immediate corollary of Proposition \ref{prop:cpe} is the following.
Let $\Pt$ be the joint distribution on $(X,\Yt)$, $\qt = \Pt(\Yt = 1)$,
$\etat(x) = \Pt(Y=1 \, | \, X = x)$, and let $\etatmax$ and $\etatmin$
denote the essential supremum and infimum of $\etat$. Further let $\Ptp$
and $\Ptn$ denote the class conditional distributions of $X$ given $\Yt =
1, 0,$ respectively.

\begin{cor}
\label{cor:cpe}
Consider the setting of the previous paragraph.
If {\bf (A)} and {\bf (C)} hold, then
\begin{equation}
\label{eqn:pi0cpe}
\pi_0 = \frac{\etatmin(\etatmax-\qt)}{\qt(\etatmax - \etatmin)}
\end{equation}
and
\begin{equation}
\label{eqn:pi1cpe}
\pi_1 = \frac{(1-\etatmax)(\qt-\etatmin)}{(1-\qt)(\etatmax - \etatmin)}.
\end{equation}
\end{cor}
\begin{proof}
By Proposition \ref{prop:cpe}, we have that
$$
\etatmax = \frac{1}{1+\frac{1-\qt}{\qt} \ks(\Ptp|\Ptn)}
$$
and
$$
\etatmin = 1 - \frac{1}{1+\frac{\qt}{1-\qt} \ks(\Ptn|\Ptp)}.
$$
The result now follows from these equations, Corollary
\ref{cor:suff}, and algebra.
\end{proof}

Note that $\qt$ is easily estimated from the fraction of contaminated
training examples with $\Yt = 1$. Therefore, estimates of $\etatmax$ and
$\etatmin$ lead directly to estimates of the contamination proportions
$\pi_0$ and $\pi_1$. This approach to estimating label noise proportions
is
explored experimentally below, where it is compared with the ROC-based
estimator.

\citet{menon15icml} adopt the conditions $\emax = 1$ and $\emin = 0$
together with {\bf (A)} as their identifiability conditions for label
noise under the contamination model. From the above discussion, these
conditions are clearly equivalent to ours. \citet{liutao16} consider the
label flipping model for label noise. They consider an equivalent
sufficient condition based on $\eta$ in that context. Connections with our
mutual irreducibility assumption are noted in each of these works.

\section{Implementation of Estimators}
\label{sec:imp}

The ROC characterization from Proposition \ref{prop:ks} says that $\ks$ is
the minimum slope of any line passing through the point $(1,1)$ in ROC
space and any other point on the optimal ROC. If the optimal ROC happens
to be concave, this is the slope of the ROC at its right end-point. See
Fig. \ref{fig:roc}.

\begin{figure}
\centering
\includegraphics[width=\textwidth]{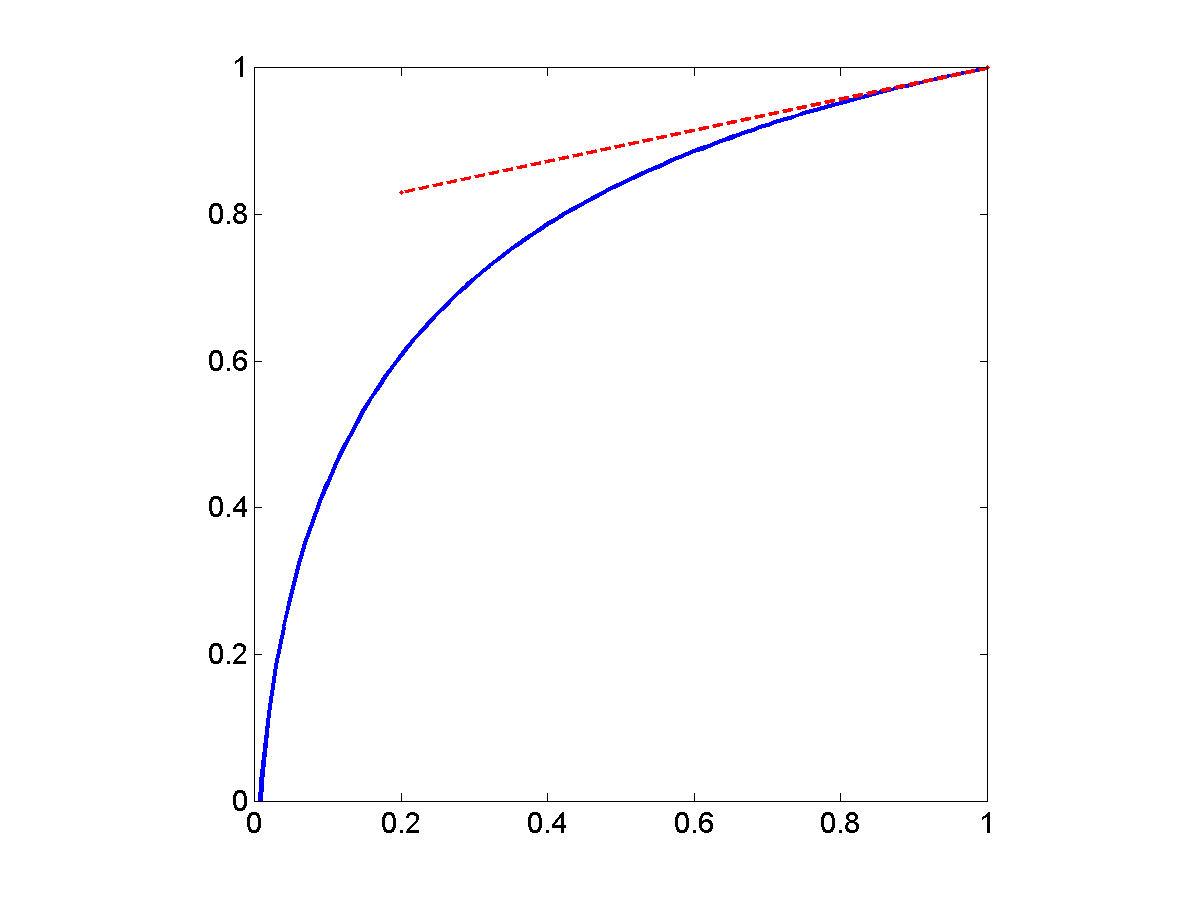}
\caption{$\ks(F|H)$ is the slope of the optimal receiver operating
characteristic for testing $H_0: X \sim H$ versus $H_1: X \sim F$ at its
right endpoint.
\label{fig:roc}}
\end{figure}

Motivated by this idea, we suggest the following practical algorithm for
MPE. First, split each of the two samples \eqref{eqn:contam0} and
\eqref{eqn:contam1} into two portions according to a common ratio.
Using the first portion of
each data set, run a universally consistent classification algorithm that
yields a full ROC. In our implementation, we run kernel logistic
regression (KLR) with a Gaussian kernel, and vary the threshold on the
posterior probability estimate to obtain an ROC. Note that KLR is run on
the contaminated data. The bandwidth and regularization parameters of KLR
are set using cross-validation.

Using the second half of each sample, construct conservative estimates (as in
Eqn. \eqref{eqn:khat3}) of the ROC for a discrete set of thresholds on the
KLR posterior probability function. To obtain these conservative estimates,
we do not use the empirical error plus or minus a VC bound. Instead, we use
direct binomial tail inversion (also known as one-sided exact Clopper-Pearson
confidence interval), which is the tightest possible deviation
bound for a binomial random variable \citep{langford05tut}. Using these
conservative estimates, we then compute the minimum slope of all line
segments joining points on the ROC to the point $(1,1)$.

We also considered an alternative approach to estimating the label noise
proportions, based on class probability estimation as discussed in Section
\ref{sec:cpe}. As in the preceding estimator, we split each sample into
two portions, and used the first portion of each sample to train a KLR
estimate of the class probability function $\etat$. We then used the
second portion of each sample to estimate the minimum and maximum values
of $\etat$, which we then plugged into the formulas
\eqref{eqn:pi0cpe}-\eqref{eqn:pi1cpe} to obtain estimates of $\pi_0$ and
$\pi_1$. To obtain some robustness to outliers, we estimated the maximum
and minimum using the 99th and 1st percentiles, respectively, as suggested
by \cite{menon15icml}.

The estimates based on the ROC method are denoted $\pnhatroc$ and
$\pphatroc$, while the estimates based on class probability estimation are
denoted $\pnhatcpe$ and $\pphatcpe$. The former estimates are based on a
20/80 split of each sample, and the latter on a 80/20 split, as these seemed
to give the best results. The latter ratio was also employed by
\citet{menon15icml}. A detailed Matlab implementation, which reproduces our
results, can be downloaded from {\tt http://web.eecs.umich.edu/$\sim$cscott}.

\section{Experiments}
\label{sec:exp}

To study the performance of the above estimators, we examined the problem
of classification with label noise using three data sets. The waveform
data set is available from the UCI Repository, and consists of three
classes of synthetically generated waveforms. The classes are overlapping,
as the Bayes risk for this data set is known to be around 10 \%. We
generated data for a binary classification problem (using only two of the
classes) with label noise proportions $\pi_0$ and $\pi_1$ specified as in
Table \ref{tab:ln}. Sample sizes of $n_0 = n_1 = 1000$ were chosen. We
also used the MNIST handwritten digits data set, digits 3 and 8, with a
similar setup as to the waveform data. In this case the sample sizes were
$n_0 = n_1 = 2000$.

A third data set comes from nuclear particle classification, where the
training data are realistically described by the label noise model. The
data are obtained from organic scintillation detectors, which detect both
gamma-rays and neutrons, and associate every detected particle with a
digitally sampled pulse-shaped waveform \citep{adams78}. The goal is to
classify gamma-ray pulses (class 0) from neutron pulses (class 1).
See discussion in
Section~\ref{sec:motiv}. Training data were obtained by measuring particles
emitted from a Cf-252 source, which undergoes spontaneous decay and emits
both neutrons and gamma rays. Data were preprocessed by aligning pulse
peaks and by eliminating signals with multiple peaks (corresponding to
multiple detected events within a single observation window). Through a
special experimental configuration \citep{ambers11}, the time of flight
(TOF) for each particle hitting the detector was also measured. Since
neutrons travel more slowly than gamma-rays, this gives noisy labels by
looking only at those particles with TOF in a certain window. Gamma-rays
travel at the speed of light, so a data set with mostly gamma-ray pulses
was obtained by focusing on those particles with TOFs around the speed of
light (TOF $<$ 5 ns). However, neutrons can still have TOFs in this window
because they were generated from either a background event or from another
fission event that occurred just an instant before the one being measured.
A neutron TOF-window was also selected (45 $<$ TOF $<$ 55 ns), and as with
the other window, this one will also contain some proportion of gamma-ray
pulses. We obtained samples of size $n_0 = n_1 = 3000$ from each window.
It is important to keep in mind that in this application, the ground truth
$\pi_0$ and $\pi_1$ are unknown, and it can only be assessed whether our
estimates of these quantities are reasonable based on physics knowledge.

The results are reported in Table \ref{tab:ln}. Regarding the ROC method,
the results indicate that this method provides reasonably accurate
estimates of the label noise proportions in the four experimental settings
where the true proportions are known. These results also suggest that
mutual irreducibility can be a reasonable assumption in practice. In the
nuclear particle classification problem, although ground truth labels are
unavailable, the proportions estimated by the ROC method are at least
consistent with the expectation that noisy labels should be relatively
rare (given the high rate of Cf-252 fission events relative to the
expected rate of background events), and also with the knowledge that
neutrons are rarer background events than gamma-rays (i.e., $\pi_0 <
\pi_1$).

\begin{table}
\centering
\begin{tabular}{|l|c|c|c|c|c|c|}
\hline
data set & $\pi_0$ & $\pi_1$ & $\pnhatroc$ & $\pphatroc$ & $\pnhatcpe$ &
$\pphatcpe$ \\
\hline
waveform & 0.1 & 0.25 & 0.0979 & 0.2792 & 0.1919 & 0.1393 \\
\hline
waveform & 0.15& 0.05 & 0.1437 & 0.0589 & 0.0369 & 0.0831\\
\hline
digits & 0.1 & 0.25 & 0.1325 & 0.2573 & 0.1065 & 0.0555 \\
\hline
digits & 0.15 & 0.05 & 0.1633 & 0.0597 & 0.0191 & 0.0479 \\
\hline
nuclear & N/A & N/A & 0.0100 & 0.0641 & 0.0151 & 0.0007\\
\hline
\end{tabular}
\caption{Results for mixture proportion estimation as applied to
classification with label noise. \label{tab:ln}}
\end{table}

With regards to the CPE method, the results indicate that the method is
sometime accurate, but other times incurs considerable error. We also note
that for the nuclear data, $\pi_1$ is estimated to be smaller than $\pi_0$,
which is inconsistent with the knowledge that contaminating neutrons are more
rare than contaminating gamma-rays. To further investigate this issue, we
formed Table~\ref{tab:cpe}. The first two columns are the same as in the
previous table, restricted to the waveform and digits data for which ground
truth is known. The third and fourth columns show the empirical percentiles
of the ground truth values of $\etatmin$ and $\etatmax$, which should ideally
be near 0 and 1. 

We see that our implementation of the CPE estimator can be both 
conservative (estimating more noise than is actually present), and overly 
optimistic (estimating less noise than is present). Indeed, percentiles 
far from 0 or 1 reflect over optimism. On the other hand, percentiles of 
exactly 0 and 1 (of which there is one instance in Table~\ref{tab:cpe}) 
are quite likely signs of conservitism. In this case, the empirical values 
of $\etat$ do not cover the full range $[\etatmin,\etatmax]$.

CPE-based estimators were also studied by \citet{liutao16,menon15icml},
who report more favorable results for this method. We use KLR to estimate
the class probabilities, whereas they employ different techniques. Given
this discrepancy in findings, the issue warrants further investigation.
There are two factors that may favor the ROC-method. First, the ROC method
employs uncertainty quantification (on the deviation between true and 
empirical probabilities) in the form of direct binomial tail
inversion when estimating the slope of the ROC at its right endpoint.
Similar uncertainty quantification would likely benefit the CPE method and
make it less overly optimistic. Second, the ROC method leverages the shape
constraint that it is typically concave.

\begin{table}
\centering
\begin{tabular}{|l|c|c|c|c|}
\hline
data set & $\pi_0$ & $\pi_1$ & $\etatmin$ \%ile & $\etatmax$ \%ile \\
\hline
waveform & 0.1 & 0.25 & 0 & 0.79 \\
\hline
waveform & 0.15& 0.05 & 0.28 & 0.99 \\
\hline
digits & 0.1 & 0.25 & 0.07 & 0.70\\
\hline
digits & 0.15 & 0.05 & 0.31 & 0.93 \\
\hline
\end{tabular}
\caption{Percentiles of the true $\etatmin$ and $\etatmax$ for those
experiments with ground truth. The 3rd and 4th columns should ideally be
zero and 1.
\label{tab:cpe}}
\end{table}

To illustrate the importance of accounting for label noise, we further
examine nuclear particle classification. As noted in Section
\ref{sec:challenge}, training a classifier on contaminated training data
generates the same ROC as training with uncontaminated data, and the real
impact of accounting for label noise occurs in performance evaluation. In
Fig. \ref{fig:corrected_roc}, the solid curve plots the ROC for the
nuclear particle data, using contaminated test data to estimate the false
positive and true positive rates. The dotted curve relies on Eqns.
\eqref{eq:noisyz2}-\eqref{eq:noisyo2} to correct these probabilities,
revealing that the classifier actually classifies the particles much more
accurately than one would expect if label noise was not accounted for.
This makes intuitive sense, because many of the particles from the
contaminated test data that appear to be incorrectly classified are
actually correctly classified, and just have erroneous labels.

\begin{figure}
\centering
\includegraphics[width=\textwidth]{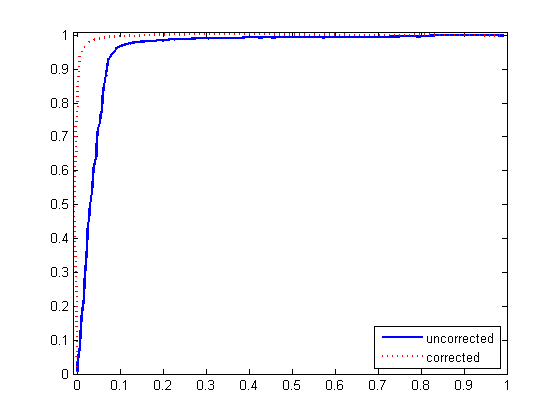}
\caption{The ROC for the nuclear classification problem, where the solid
curve plots the uncorrected errors $(\Rtn, 1-\Rtp)$, and the dotted curve
plots the corrected errors $(R_0, 1-R_1)$ which account for the presence
of label noise in the test data.
\label{fig:corrected_roc}}
\end{figure}

\section{Conclusion}
\label{sec:conclusion}

We argue that consistent classification with label noise is possible
if a majority of the labels are correct on average, and the
class-conditional distributions $P_0$ and $P_1$ are mutually irreducible.
Under these conditions, we leverage results of \cite{blanchard10} on
mixture proportion estimation to design consistent estimators of the noise
proportions. These estimators are applied to establish a consistent
discrimination rule based on surrogate loss minimization, although other
performance measures could be analyzed similarly.  Unlike previous
theoretical work on this problem, we handle the cases where the supports
of $P_0$ and $P_1$ may overlap or even be equal, and the noise proportions
are asymmetric and unknown.

We also argue that mutual irreducibility is necessary if we require the
decontamination operation at population level to satisfy some natural
conditions (universality, symmetry, continuity and stability.)
Additionally, requiring mutual irreducibility can be equivalently seen as
aiming at maximum denoising of the contaminated distributions, or maximum
separation of the unknown sources $P_0,P_1$ for given contaminated
distributions. Thus, our discrimination rule is universally consistent in
the sense that its performance tends to the optimal performance
corresponding to the maximally denoised $P_0, P_1$, regardless of $\Ptn,
\Ptp$.

Finally, we investigate two practical implementations of MPE, one based on
the ROC for the contaminated data, and the other based on class probability
estimation for the contaminated data. The ROC method exhibits good accuracy
in the label noise setting on three different data sets, including the
nuclear particle classification problem that originally motivated this work.
Our CPE implementation, on the other hand, still requires further
development.

\section*{Acknowledgements}

G. Blanchard was supported in part by the European Community's 7th
Framework Programme under the E.U. grant agreement 247022 (MASH Project),
and by the DFG grant FOR-1735 (Structural inference in statistics).
C. Scott was supported in part by NSF Grants 0953135, 1047871, and
1217880.

\appendix

\section{Mixture Proportion Consistency Result}
\label{app:mpe}

\citet{blanchard10} establish strong consistency of $\widehat{\kappa}$,
that is, convergence almost surely, although the statement of that
consistency result requires a slight correction. In particular, it is
necessary to additionally assume that $\log \max(n_0,n_1) =
o(\min(n_0,n_1))$ for the argument to hold. Although the focus of that
work is almost sure convergence, the proof can be easily modified to
establish convergence in probability, and for that type of convergence,
the aforementioned qualification on the growth of the sample sizes is not
necessary. Since the present work focuses on convergence in probability,
our results also require no additional qualification. See
\cite{scottWSLnotes} for additional details.

\section{Remaining Proofs}

\subsection{Proof of Proposition~\ref{prop:p1}}
\begin{proof}
First note that under {\bf (A)}, $\lam$ is well-defined and nonnegative.
Solving for $\g$ we obtain
$$
\g = \frac{\lam (1-\pn) - \pp}{1 - \pp - \lam \pn}.
$$
The denominator in this expression is positive, which can be seen as follows.
\begin{align*}
\lam & = \frac{\pp + \g (1-\pp)}{1 - \pn + \g \pn} \\
& < \frac{1 - \pn + \g (1-\pp)}{1 - \pn + \g \pn} \\
& < \frac{\g (1-\pp)}{\g \pn} \\
& = \frac{1-\pp}{\pn}.
\end{align*}
The first inequality follows from {\bf (A)}, while the second follows from the fact that the mapping $t \mapsto (a + t)/(b + t)$ is strictly decreasing in $t \ge 0$ when $a > b$.  Here $a = \g (1-\pp)$ and $b = \g \pn$.

Therefore,
\begin{align*}
\frac{p_1(x)}{p_0(x)} > \g & \iff \frac{p_1(x)}{p_0(x)} > \frac{\lam
(1-\pn) - \pp}{1 - \pp - \lam \pn} \\
& \iff [1 - \pp - \lam \pn] p_1(x) > [\lam (1-\pn) - \pp] p_0(x) \\
& \iff (1-\pp) p_1(x) + \pp p_0(x) > \lam [(1-\pn) p_0(x) + \pn p_1(x)] \\
& \iff \frac{\htp(x)}{\htn(x)} > \lam.
\end{align*}
\end{proof}

\subsection{Proof of Theorem~\ref{thm:necessity}}

\begin{proof}
Let $(P_0,P_1)$ be mutually irreducible and fixed for the rest of the proof.
Observe that if conditions (i) and (ii) are satisfied, $(P_0,P_1)$ must be a
$\phi$-source. Namely, by (i) $(P_0,P_1)$ belongs to the domain of $\phi$;
and since $(P_0,P_1)$ are mutually irreducible,
the only possible values for $\phi(P_0,P_1)$
are $(0,0,P_0,P_1)$ and $(1,1,P_1,P_0)$.
In any case, by the symmetry condition (ii), $(P_0,P_1)$
and $(P_1,P_0)$ are both $\phi$-sources. Finally, by stability condition (iv),
it must be the case that $\phi(P_0,P_1) = (0,0,P_0,P_1)$.

Let us now denote $\eps^*$ the supremum of values $\eps$ such that,
for all $(\pi_0,\pi_1) \in [0,1]^2 \text{ with } \pi_0 + \pi_1 < \eps$,
\eqref{eq:stab} is satisfied. Condition (iv) implies $\eps^*>0$.
If $\eps^*=1$, this means that $\phi$ returns the mutually irreducible solution
for mutual contamination with arbitrary contamination weights of $(P_0,P_1)$.

We now consider the case where $\eps^*<1$ and will come to a contradiction.
Fix arbitrarily $\eps \in (\eps^*,1)$.
By definition of $\eps^*$, there exists $(\pi_0,\pi_1)$ such
that $\eps^* < \pi_0 + \pi_1<\eps$ and \eqref{eq:stab}
is not satisfied.

Let $(\nu_0,\nu_1,P^\eps_0,P^\eps_1)
= \phi(\psi(\pi_0,\pi_1,P_0,P_1))$ be the contamination proportions
and sources identified by $\phi$ for the contamination
$\psi(\pi_0,\pi_1,P_0,P_1)$. Since \eqref{eq:stab} is not satisfied,
and the identified sources uniquely determine the associated contamination
proportions, it holds that $(P^\eps_0,P^\eps_1)$ is a $\phi$-source
distinct from $(P_0,P_1)$. Finally let us denote $(\eta_0,\eta_1)$
the contamination weights of $(P^\eps_0,P^\eps_1)$ in its mutually irreducible
decontamination in terms of $(P_0,P_1)$. (Observe
that $P^\eps_0$ and $P^\eps_1$ both belong to the convex hull
of $P_0$ and $P_1$; this implies that $(P^\eps_0,P^\eps_1)$
decontaminate irreducibly either to
$(P_0,P_1)$ or to $(P_1,P_0)$. We assume for now the former case
and will come back to the latter case later.)

It must hold that $\eta_0+\eta_1\geq \eps^*$, otherwise we would
have (by definition of $\eps^*$) $\phi(P^\eps_0,P^\eps_1) =
(\eta_0,\eta_1,P_0,P_1) = \phi(\psi(0,0,P^\eps_0,P^\eps_1))$,
contradicting (iv) for the source $(P^\eps_0,P^\eps_1)$.
Moreover, straightforward computations and coefficient
identification in the unique representation in terms of
$(P_0,P_1)$ lead to the relations
\[
(\nu_0,\nu_1) = \paren{ \frac{\pi_0 - \eta_0}{1-
(\eta_0+\eta_1)}, \frac{\pi_1 - \eta_1}{1-
(\eta_0+\eta_1)}}.
\]
It follows that
\[
\nu_0+\nu_1 = 1- \frac{1-(\pi_0+\pi_1)}{1-(\eta_0+\eta_1)}
\leq 1 - \frac{1-\eps}{1-\eps^*}.
\]
In the case where $(P^\eps_0,P^\eps_1)$ decomposes irreducibly to $(P_1,P_0)$,
the first equality above still holds when replacing $\eta_i$ by $(1-\eta_i)$.
We deduce that in that case $\nu_0+\nu_1> 1$.

Now consider a sequence $\eps_n \searrow \eps^*$, and the associated
sequences $(\pi_0^{(n)},\pi_1^{(n)})$ and $(\eta_0^{(n)},\eta_1^{(n)})$
constructed
as above. By compactness, we can extract a subsequence so that
$(\pi_0^{(n)},\pi_1^{(n)})$ converges to some $(\pi_0^*,\pi_1^*)$.
Then by construction $\pi_0^*+\pi_1^*=\eps^* \in (0,1).$ On the other hand,
for all $n$ either $\nu_0^{(n)}+\nu_1^{(n)}\leq 1 - \frac{1-\eps_n}{1-\eps^*}$
(which gets arbitrarily close to 0) or $\nu_0^{(n)}+\nu_1^{(n)}\geq 1$.
This contradicts the continuity assumption (iii) at point $(\pi_0^*,\pi_1^*)$,
since by definition of $\eps^*$ and (iii), it should hold
$\phi_\pi(\psi(\pi_0^*,\pi_1^*,P_0,P_1)) = (\pi_0^*,\pi_1^*)$ and
thus we should have $\nu_0^{(n)}+\nu_1^{(n)} \rightarrow \eps^*$.


Conversely, if $\phi$ is the mutually irreducible decontamination
operator, it satisfies (i)-(iv), and is therefore the only decontamination
operator having these properties.
\end{proof}

\subsection{Proof of Theorem~\ref{thm:cplt}}
\begin{proof}
By Lemmas~\ref{le:le1} and~\ref{le:le1conv}, feasible quadruples
$(\pn,\pp,P_0,P_1)$ for decompositions
\eqref{eqn:contam0}-\eqref{eqn:contam1}\, under
condition {\bf (A)} are in one-to-one correspondence with feasible
quadruples $(\ptn,\ptp,P_0,P_1)$ for decompositions
\eqref{eqn:ssnd0}-\eqref{eqn:ssnd1}\,.


Define $\ptn^* := \ks(\Ptp|\Ptn)$.
Proposition \ref{prop:canondecmp}\, applied to \eqref{eqn:ssnd0}\,
easily implies that
for any value $\ptn \in [0, \ptn^*]$\,, there exists a unique
$P_0$ such that $(\ptn,P_0)$ satisfies \eqref{eqn:ssnd0}; also,
the solution $(\ptn^*, P_0^*)$ corresponding to the maximal feasible value
of $\ptn$ is the unique one satisfying \irreducible.
A similar conclusion is valid concerning solutions of \eqref{eqn:ssnd1}.

Therefore, the feasible region $R$ for proportions $(\pn,\pp)$ in the
original model \eqref{eqn:contam0}-\eqref{eqn:contam1}\, is obtained as
the image of the rectangle $[0,\ptn^*]\times[0,\ptp^*]$ via the above
one-to-one correspondence. Using the explicit expression for $(\ptp,\ptn)$
of Lemma~\ref{le:le1}, the constraints \eqref{eqn:feas} simply translate
the equivalent constraints $\ptn\leq \ptn^*,\, \ptp \leq \ptp^*$.

Since by Lemma~\ref{le:le2}\,, under the assumption {\bf (A)} conditions
\irreducible\, and \mutual\, are equivalent, then again via the above
correspondence, we get existence and unicity of
$(\pn^*,\pp^*,P_0^*,P_1^*)$ for the original formulation
\eqref{eqn:contam0}-\eqref{eqn:contam1}, under condition \mutual. The
explicit expression \eqref{eqn:expl} for $(\pn^*,\pp^*)$ is obtained
via Lemma~\ref{le:le1conv}.


The equality $\pn+\pp=1-\frac{(1-\ptp)(1-\ptn)}{1-\ptp\ptn}$
implies that $\pn+\pp$ is a monotone (strictly) increasing function of
$\ptp$ and $\ptn$. Therefore, the maximum of $\pn+\pp$ can only be
reached when both $(\ptp,\ptn)$ take their maximum value. Since the latter
values are attained for the unique feasible quadruple
$(\ptn^*,\ptp^*,P_0^*,P_1^*)$ in the decoupled problem, the corresponding
maximum of $\pn+\pp$ for the original formulation is also uniquely
attained for the quadruple $(\pn^*,\pp^*,P_0^*,P_1^*)$.

Finally, by subtracting \eqref{eqn:contam0} from \eqref{eqn:contam1}, we obtain the relation
\[
(P_1-P_0) = (1 - \pn- \pp)^{-1} (\Ptp - \Ptn)\,
\]
implying
\[
\norm{P_1 - P_0}_{TV} =  (1 - \pn - \pp)^{-1} \norm{\Ptp - \Ptn}_{TV}.
\]
Therefore, the maximum (over $\Lambda$) of the total variation distance $\norm{P_1 - P_0}_{TV}$
is precisely attained for the maximum value of $(\pn + \pp)$, and hence corresponds to the unique
mutually irreducible solution.
\end{proof}

\subsection{Alternate Proof of Density Ratio Formula for $\ks$}

\begin{prop}
Assume that the ROC of the likelihood ratio tests $x \mapsto
\ind{f(x)/h(x) > \gamma}$ is left-differentiable at $(1,1)$. Then
$\ks(F|H)$ is the slope (left-derivative) of the ROC at $(1,1)$.
\end{prop}
\begin{proof}
The slope of the ROC of an LRT with threshold $\g$ is equal to $\g$
wherever the slope is well defined \citep{birdsall54,scharf91}. The right
end-point of the ROC
corresponds to $\g^* = \essinf_{x  \in \supp(H)} \frac{f(x)}{h(x)}$. That
is, for all $\g > \g^*$, the Type I error of the LRT is strictly less than
1, whereas it equals 1 at $\g^*$.
\end{proof}

\subsection{Proof of Theorem \ref{thm:mperate}}

We begin by establishing \eqref{eqn:mperate} without the absolute value, which
is the more challenging direction. The reverse direction will follow easily by
the first part of Theorem \ref{thm:consist}.

By {\bf (D)}, there exists a distribution $G$ and $\gamma \in [0,1]$ such
that $F= (1-\gamma)G + \gamma
H$ and $\supp(H) \not\subset \supp(G)$. Then $G$ is irreducible with respect to
$H$, and Corollary \ref{cor:irrd} implies that $\gamma = \ks$. By {\bf
(AP2)}, there exists $j \ge 1$ and $S \in \sS_j$ such that $G(S) = 0$ and
$H(S)
> 0$. But then
$$
\frac{F(S)}{H(S)} = (1-\gamma)\frac{G(S)}{H(S)} + \gamma = \ks.
$$
By the VC inequality and union bound, we have that with
probability at least $1-2(\frac1{n_0}+ \frac1{n_1})$,
$$
\khat \le \frac{F(S) + 2\eps_1(j,j^{-2}/n_1)}{(H(S) -
2\eps_0(j,j^{-2}/n_0))_+} \le
\frac{F(S) + \eps}{(H(S) - \eps)_+}
$$
where $\eps := 2(\eps_1(j,j^{-2}/n_1) + \eps_0(j,j^{-2}/n_0))$. Now let $\nu$ be
such that $\eps = \frac{\nu}{1+\nu}H(S)$, which is achieved by $\nu =
\frac{\eps}{H(S)-\eps}$. Let $N$ be such that $n_0, n_1 \ge N$ implies $\eps
\le \frac12 H(S)$. Then, for $n_0, n_1 \ge N$ and with probability at least
$1-2(\frac1{n_0}+ \frac1{n_1})$,
\begin{align*}
\khat &\le (1+\nu)\frac{F(S)+\eps}{H(S)} \\
&= (1+\nu)\ks + \nu \\
&\le \ks + 2 \nu \\
&\le \ks + \frac{4}{H(S)} \eps.
\end{align*}
This establishes the existence of a constant $C$ such that for $n_0, n_1 \ge N$,
$$
\Pr \left( \khat - \ks \ge C \left[\sqrt{\frac{\log n_0}{n_0}} +
\sqrt{\frac{\log n_1}{n_1}} \right] \right) \le \frac2{n_0} + \frac2{n_1}.
$$
The same inequality holds with the absolute value by the first part of Theorem
\ref{thm:consist}, which holds on the same event (samples where the VC bounds
hold for
all $k \ge 1$) as was used to establish the above inequality.

\subsection{Proof of Theorem \ref{thm:discrim}}
By Corollary \ref{prop:regret}, it suffices to show $R_{\Pt,L_\a}(\fhat) -
R_{\Pt,L_\a}^* \to 0$ in probability. Toward this end we employ Rademacher
complexity analysis. In particular, we will leverage
the following result.
\begin{thm}
Let $Z, Z_1, \ldots, Z_n$ be iid random variables taking values in a set
$\sZ$. Let $\sigma_1, \ldots, \sigma_n$ be iid Rademacher random
variables, independent of  $Z, Z_1, \ldots, Z_n$. Consider a set of
functions $\mathcal{G} \subseteq
[a,b]^{\mathcal{Z}}$.
$\forall \delta>0$, with probability $\ge 1-\delta$ with respect
to the draw of $Z_1, \ldots, Z_n$, we have
\begin{align}
\label{RC1}
\forall g \in \mathcal{G}, \ \left| \mathbb{E}[g(Z)] -
\frac{1}{n}\sum_{i=1}^n
g(Z_i) \right| \le 2 \mathfrak{R}_n(\mathcal{G}) + (b-a) \sqrt{\frac{\log
2/\delta}{2n}},
\end{align}
where
$$
\mathfrak{R}_n(\mathcal{G}) = \bbE_{\stackrel{Z_1, \ldots,
Z_n}{\sigma_1, \ldots, \sigma_n}} \left[ \sup_{g \in
\mathcal{G}} \frac{1}{n} \sum_{i=1}^n \sigma_i g(Z_i)\right]
$$
is the Rademacher complexity of $\mathcal{G}$.
\end{thm}
\noindent A proof may be found in \citet[Thm. 3.1]{mohri12}.

For any $f \in \sH$ and loss $L'$, denote
the empirical $L'$-risk
$$
\wR_{L'}(f) := \frac1{n} \sum_{i=1}^n L'(\Yt_i, f(X_i)),
$$
and denote the objective function $J(f): = \wR_{L_{\wa}}(f) + \lambda_n \|f\|^2$.
Also
define $C_0 := \max\{L(0,0),L(1,0)\}$. Observe that $J(\fhat) \leq J(0)
\leq
C_0$. Therefore $\lambda_n
\|\fhat\|^2 \leq C_0 - \wR_{L_{\wa}}(\fhat)\leq C_0$, and we deduce that
$\fhat
\in
B_{\sH}(M_n)$, the ball of radius $M_n$ in $\sH$, where $M_n:=
\sqrt{C_0/\lambda_n}$.

Let $\eps > 0$, and let $f_\eps \in \sH$ be such that
$R_{\Pt, L_\a}(f_\eps) < R_{\Pt,L_\a}^* + \frac{\epsilon}{2}$, which is possible
since the
the reproducing kernel associated with $\sH$ is universal
\citep{steinwart08}.
Then
\begin{align*}
R_{\Pt, L_{\a}}(\fhat) - R_{\Pt,L_{\a}}(f_\eps)
&= R_{\Pt, L_{\a}}(\fhat) - \wR_{L_{\a}}(\fhat) \\ 
&+ \wR_{L_{\a}}(\fhat) - \wR_{L_{\wa}}(\fhat) \\ 
&+ \wR_{L_{\wa}}(\fhat) - \wR_{L_{\wa}}(f_\eps) \\ 
&+ \wR_{L_{\wa}}(f_\eps) - \wR_{L_{\a}}(f_\eps) \\ 
&+ \wR_{L_{\a}}(f_\eps) - R_{\Pt,L_{\a}}(f_\eps). \\ 
\end{align*}
The first term
can be bounded, with probability at least $1-1/n$, by
$$
\frac{2DBM_n}{\sqrt{n}} + (C_0 + DBM_n) \sqrt{\frac{\ln 2n}{2n}}
$$
using the Rademacher complexity bound applied to the class of functions
$\sG = \{(x,\tilde{y}) \mapsto L(\tilde{y},f(x)), f \in B_{\sH}(M_n)\}$,
where
$B_{\sH}(M_n)$
is
the ball of radius $M_n$ (centered at the origin) in $\sH$. By the
Lipschitz composition property of Rademacher complexity
\citep[Lemma 4.2]{mohri12}, $\mathfrak{R}_n({\sG}) \le D
\mathfrak{R}_n(B_{\sH}(M_n))$.
The
Rademacher complexity of $B_{\sH}(M_n)$ is further bounded by $B M_n /
\sqrt{n}$ \citep[Thm 5.5]{mohri12}, which gives the first term on the RHS.
The
second term comes
from the observation that functions in $\sG$ have ranges confined to
$[0,C_0 + DBM_n]$. To see this, recall that losses are by definition
nonnegative, that $L$ is Lipschitz in its second argument, and that for
any $f \in B_{\sH}(M_n)$, we have $\|f \|_\infty = \sup_{x \in \sX}
|\langle
f, k(\cdot,x) \rangle | \le B M_n$ by the reproducing property and
Cauchy-Schwarz.

The fifth term is bounded similarly, with the only additional
observation being that $f_\eps \in B_{\sH}(M_n)$ for $n$ sufficiently
large.

The middle term can be bounded by $\lambda_n \| f_\eps \|^2$, which tends to zero
as $n \to \infty$. This follows from the definition of $\fhat$, since $J(\fhat)
\le J(f_\eps)$ implies $\wR_{L_{\wa}}(\fhat) - \wR_{L_{\wa}}(f_\eps) \le
\lambda_n \| f_\eps \|^2 -
\lambda_n \| \fhat \|^2 \le \lambda_n \| f_\eps \|^2$.

To bound the second term, observe that for any $f \in B_{\sH}(M_n)$,
\begin{align*}
\wR_{L_{\a}}(f) - \wR_{L_{\wa}}(f)
&= \frac1{n} \Bigg[ \sum_{i:\Yt_i=1}  (\wa - \a) L(1,f(X_i)) \\
& \ \ \ \ \ \ \ \ +
\sum_{i:\Yt_i=0}  (\a - \wa) L(0,f(X_i)) \Bigg] \\
&\le |\wa - \a| \sup_{x,\tilde y} L(\tilde y,f(x)) \\
&\le |\wa - \a| \left( C_0 + D \| f \|_\infty \right),
\end{align*}
where $D$ is the Lipschitz constant of $L$. By Cauchy-Schwarz and the
reproducing property,
$$
\| f \|_\infty = \sup_x |\langle f, k(\cdot,x)\rangle| \le \|f\|_{\sH} B
$$
where $B$ is the bound on the kernel. Now $\|f\|_{\sH} \le
\sqrt{\frac{C_0}{\lambda_n}}$, and so for the second term to go to zero,
we need
$|\wa - \a|/\lambda_n$ to go to zero. Under {\bf (A')} and {\bf (C')}, we
know that $|\wa - \a|$ converges at a rate of $\sqrt{\frac{\log n}{n}}$, and by
our assumption on the rate of decay of $\lambda_n$, $|\wa - \a|/\lambda_n$
tends to zero as $n \to \infty$, except on a vanishingly small event.

The fourth term is handled in a similar manner, where again we observe that
$f_\eps \in B_{\sH}(M_n)$ for $n$ sufficiently large.

In summary, we have shown that $R_{\Pt, L_{\a}}(\fhat) - R_{\Pt,L_{\a}}^* \le
\eps$ with probability tending to one as $n$ (and with it $n_0$ and $n_1$) tends
to infinity. This concludes the proof.

\subsection{Proof of Theorem \ref{thm:clippable}}

We start by establishing that $L_\alpha$ is $T$-clippable and its clipped
version is Lipschitz and bounded with constants independent of $\alpha \in (0,1)$.
The loss $L$ being $T$-clippable implies by definition that both
$L_0(t)=L(0,t)$ and $L_1(t) = L(1,t)$ are clippable. Therefore, $L_\a(y,t) = (1-\a)
\ind{y=1} L_1(t) + \a \ind{y=0} L_0(t)$ is $T$-clippable
(regardless of $\alpha \in (0,1)$.) Denote
\[
\wt{L}_\a(y,t) := L_\a(y, \clip_T(t)) =
(1-\a) \ind{y=1} L_1(\clip_T(t)) + \a \ind{y=0} L_0(\clip_T(t))\,,
\]
and define $C_0 := \max\{L(0,0),L(1,0)\}$\,.
Since $L$ is assumed Lipschitz with constant $D$, both $L_1$ and $L_0$ are
Lipschitz and since $\clip_T$ is 1-Lipschitz, by composition $\wt{L}_\a$ is also
a Lipschitz loss (with the same constant $D$, regardless of $\a \in (0,1)$.)
Furthermore, since $\clip_T(t)\in [-T,T]$\,, we have for all $(y,t)$ and $\a$:
\[
\abs{\wt{L}_\a (y,t)} \leq \max_{t \in [-T,T]} \max \paren{ L_0(t),L_1(t)} \leq C_0 + DT\,.
\]

We proceed to proving the main claim. By Corollary \ref{prop:regret},
it suffices to show $R_{\Pt,L_\a}(\cfhat) - R_{\Pt,L_\a}^* \to 0$ in probability.
For any $f$ and loss $L'$, denote by $\wR_{L'}(f)$ the empirical $L'$-risk
of $f$.
Denote the objective function $J(f): = \wR_{L_{\wa}}(f) + \lambda_n
\|f\|^2$.
Observe that $J(\fhat) \leq J(0) \leq
C_0$. Therefore $\lambda_n
\|\fhat\|^2 \leq C_0 - \wR_{L_{\wa}}(\fhat)\leq C_0$, and we deduce that $\fhat
\in
B_{\sH}(M_n)$, the ball of radius $M_n$ in $\sH$, where $M_n:=
\sqrt{C_0/\lambda_n}$.

Let $\eps > 0$, and let $f_\eps \in \sH$ be such that
$R_{\Pt, L_\a}(f_\eps) < R_{\Pt,L_\a}^* + \frac{\epsilon}{2}$, which is possible
since the
the reproducing kernel associated with $\sH$ is universal
\citep{steinwart08}.
We have
\begin{align*}
  R_{\Pt, L_{\a}}(\cfhat) - R_{\Pt,L_{\a}}(f_\eps)
& = R_{\Pt, \wt{L}_{\a}}(\fhat) - R_{\Pt,L_{\a}}(f_\eps) \\
  & = R_{\Pt, \wt{L}_{\a}}(\fhat) - \wR_{\wt{L}_{\a}}(\fhat)\\
  & \;\;\; + \wR_{\wt{L}_{\a}}(\fhat) - \wR_{\wt{L}_{\wa}}(\fhat)\\
  & \;\;\; + \wR_{\wt{L}_{\wa}}(\fhat) - \wR_{L_{\wa}}(\fhat)\\
& \;\;\; + \wR_{L_{\wa}}(\fhat) - \wR_{L_{\wa}}(f_\eps) \\ 
& \;\;\; + \wR_{L_{\wa}}(f_\eps) - \wR_{L_{\a}}(f_\eps) \\ 
& \;\;\; + \wR_{L_{\a}}(f_\eps) - R_{\Pt,L_{\a}}(f_\eps). \\ 
\end{align*}
The first and last terms can be bounded, with probability at least $1-1/n$, by
$$
\frac{2DBM_n}{\sqrt{n}} + (C_0+D\max(T,B\norm{f_\eps}_\infty)) \sqrt{\frac{\ln 4n}{2n}}
$$
using Rademacher complexity analysis as was done in the preceding proof.
Here $D$ is the Lipschitz constant for $L$
(and thus also for $\wt{L}$) and
$B$ is the bound on the kernel.
Note that since a different
loss is used for the
first and last terms, we use a union bound and thus introduce an
additional factor in the log term.

The third term equals $\wR_{L_{\wa}}(\cfhat) - \wR_{L_{\wa}}(\fhat)$ and
is nonpositive by definition of an $T$-clippable loss.

The middle (4th) term can be bounded (as in the preceding proof) by
$\lambda_n \| f_\eps \|^2$, which tends to
zero
as $n \to \infty$.

To bound the second term, observe that for any $f$,
\begin{align*}
\wR_{\wt{L}_{\a}}(f) - \wR_{\wt{L}_{\wa}}(f)
&= \frac1{n} \Bigg[ \sum_{i:\Yt_i=1}  (\wa - \a) \wt{L}(1,f(X_i)) \\
& \ \ \ \ \ \ \ \ +
\sum_{i:\Yt_i=0}  (\a - \wa) \wt{L}(0,f(X_i)) \Bigg] \\
&\le |\wa - \a| \sup_{x,\tilde y} \wt{L}(\tilde y,f(x)) \\
&\le |\wa - \a| (C_0+DT)\,.
\end{align*}

The fifth term is handled in a similar manner, but with the non-clipped loss $L$
instead of $\wt{L}$; in this case we have
\[
\wR_{\wt{L}_{\a}}(f) - \wR_{\wt{L}_{\wa}}(f)
\le |\wa - \a| \left( C_0 + D \| f_\eps \|_\infty \right).
\]

In summary, we have shown that $R_{\Pt, L_{\a}}(\fhat) - R_{\Pt,L_{\a}}^* \le
\eps$ with probability tending to one as $n$ (and with it $n_0$ and $n_1$) tends
to infinity. This concludes the proof.

\bibliographystyle{plainnat}
\bibliography{labelNoiseTR}

\end{document}